\documentclass[10pt]{amsart}

\usepackage{amsmath}
\usepackage{mathtools}
\usepackage{dsfont}
\usepackage{mathabx}

\usepackage[usenames, dvipsnames]{color}
\usepackage[colorlinks=true,
        raiselinks=true,
        linkcolor=MidnightBlue,
        citecolor=Brown,
        urlcolor=ForestGreen,
        pdfauthor={},
        pdftitle={},
        pdfkeywords={},
        pdfsubject={},
        plainpages=false]{hyperref}
\usepackage[T1]{fontenc}

\usepackage[norelsize, ruled, vlined]{algorithm2e}

\usepackage{enumitem}
\setlist{noitemsep, topsep=0cm}

\usepackage{tikz}
\usepackage{float}
\usetikzlibrary{fadings}
\usetikzlibrary{shapes.misc}
\tikzset{cross/.style={cross out, draw=black, minimum size=2*(#1-\pgflinewidth), inner sep=0pt, outer sep=0pt},cross/.default={1pt}}
\usetikzlibrary{shapes,arrows}
\usetikzlibrary{positioning}
\usetikzlibrary{plotmarks}

\usepackage{amsaddr}

\allowdisplaybreaks
\newtheorem{theorem}{Theorem}[section]
\newtheorem{proposition}[theorem]{Proposition}
\newtheorem{lemma}[theorem]{Lemma}
\newtheorem{corollary}[theorem]{Corollary}
\newtheorem{definition}[theorem]{Definition}

\theoremstyle{remark}
\newtheorem{remark}[theorem]{Remark}
\newtheorem{assumption}[theorem]{Assumption}
\theoremstyle{claim}

\theoremstyle{definition}

\newcommand{\norm}[2]{\left\lVert#1\right\rVert_{#2}}

\newcommand{\pmat}[3]{\begin{pmatrix} #1 & #2 &\cdots & #3 \end{pmatrix}}

\newcommand{\dict}[1]{d_{#1}}
\newcommand{\inprod}[2]{\left\langle#1 \; , \; #2\right\rangle}
\newcommand{\inpnorm}[1]{\norm{#1}{}}
\newcommand{\dualnorm}[1]{\norm{#1}{\linmap}'}

\newcommand{\dimension}{n}

\newcommand{\indicator}{\mathds{1}}

\newcommand{\summ}[2]{\sum_{#1}^{#2}}

\DeclareSymbolFont{symbolsC}{U}{pxsyc}{m}{n}

\DeclareMathOperator{\relinterior}{relint}

\DeclareMathOperator{\image}{image}

\DeclareMathOperator{\convhull}{conv}

\DeclareMathOperator{\rank}{rank}

\DeclareMathOperator{\EE}{\mathsf{E}}
\DeclareMathOperator{\PP}{\mathsf{P}}

\DeclareMathOperator*{\minimize}{minimize}
\DeclareMathOperator*{\sbjto}{subject \; to}

\DeclareMathOperator*{\argmin}{argmin}
\DeclareMathOperator*{\argmax}{argmax}

\DeclarePairedDelimiterX\set[1]\lbrace\rbrace{#1}

\newcommand{\define}{\coloneqq}

\newcommand{\R}[1]{\mathbb{R}^{#1}}
\newcommand{\opt}{^\ast}
\newcommand{\transp}{^\top}
\newcommand{\dsize}{K}
\newcommand{\dictionary}{D}
\newcommand{\unitdictionaryset}{\mathcal{\dictionary}}
\newcommand{\nbhood}[2]{B(#1, #2)}
\newcommand{\closednbhood}[2]{B[#1, #2]}

\renewcommand{\geq}{\geqslant}

\renewcommand{\leq}{\leqslant}

\newcommand{\hilbert}{\mathbb{H}_{\dimension}}
\newcommand{\rv}{X}
\newcommand{\repvec}{f}
\newcommand{\linmap}{\phi}
\newcommand{\error}{\epsilon}
\newcommand{\regulizer}{\delta}
\newcommand{\scaling}{r}
\newcommand{\atomset}[1]{S_{\regulizer}(\linmap, #1)}
\newcommand{\atomball}{S_{\regulizer}(\linmap, 1)}
\newcommand{\atomscaled}{S_{\regulizer}(\linmap, \samplevec, \error)}
\newcommand{\horizon}{T}
\newcommand{\data}{[\rv : \horizon]}
\newcommand{\sample}[1]{x_{#1}}
\newcommand{\samplevec}{x}
\newcommand{\dualvar}{\eta}
\newcommand{\hvar}{h}
\newcommand{\separator}{\lambda}
\newcommand{\separatorset}{\Lambda_{\regulizer} (\linmap, \samplevec, \error)}
\newcommand{\Depsdelta}{(\linmap, \error, \regulizer)}

\newcommand{\encodermap}[4]{F_{#4} (#1, #2, #3)}
\newcommand{\codes}{\encodermap{\linmap}{\sample}{\error}{\regulizer}}
\newcommand{\encodedcost}[4]{C_{#4} (#1, #2, #3)}
\newcommand{\samplecost}{\encodedcost{\linmap}{\sample}{\error}{\regulizer}}
\newcommand{\cost}{c}
\newcommand{\costatomset}{V_{\cost}}
\newcommand{\horder}{p}

\title[On Convex Duality of Linear Inverse Problems]{On Convex Duality of Linear Inverse Problems}

\author[M. Rayyan Sheriff and D. Chatterjee]{Mohammed Rayyan Sheriff, Debasish Chatterjee}
\address{Systems \& Control Engineering\\ Indian Institute of Technology Bombay\\ Powai, Mumbai 400076\\ India.\\
\indent \url{http://www.sc.iitb.ac.in/~mohammedrayyan}\\
\indent \url{http://www.sc.iitb.ac.in/~chatterjee}
}
\thanks{Emails: \texttt{mohammedrayyan@sc.iitb.ac.in, dchatter@iitb.ac.in}}

\begin{document}
\maketitle 

\begin{abstract}
In this article we dwell into the class of so called ill posed Linear Inverse Problems (LIP) in machine learning,
which has become almost a classic in recent times. The fundamental task in an LIP is to recover the entire signal / data from its relatively few random linear measurements. Such problems arise in variety of settings with applications ranging from medical image processing, recommender systems etc. We provide an exposition to the convex duality of the linear inverse problems, and obtain a novel and equivalent convex-concave min-max reformulation that gives rise to simple ascend-descent type algorithms to solve an LIP. Moreover, such a reformulation is crucial in developing methods to solve the dictionary learning problem with almost sure recovery constraints.
\end{abstract}

\section{Introduction}
A Linear Inverse Problem can be simply stated as the task of recovering a signal from its linear measurements. Increasingly, a signal encountered in practise has very high ambient dimension. For example, audio signals and images typically have ambient dimension ranging from a few thousands to millions. However, the number of linear measurements that are available to recover the entire signal from, are relatively very few compared to their ambient dimension. This makes such an LIP ill posed. Fortunately, high dimensional data of present day and age often has a lot of underlying low dimensional characteristics to it. Such low dimensional characteristics when taken into consideration in an LIP, often suffices to overcome the ill posedness of the problem.

One of the early instances that gave recognition to LIPs is Compressed Sensing \cite{donoho2006compressed}, \cite{candes2008introduction}, \cite{candes2006robust}, \cite{candes2006stable} where a given signal \( \repvec' \) is assumed to be sparse or sparse in some known basis. So, given the partial information of the signal in the form of a collection of random linear measurements \( \samplevec = \linmap (\repvec') \), the objective is to recover the entire signal almost accurately. Since the given signal is known to be sparse, one would expect that the true signal can be recovered accurately by finding a sparsest solution to the under determined system of linear equations \( \samplevec = \linmap (\repvec) \) given by the linear measurements. However, finding sparsest solutions (i.e., minimum \( \ell_0 \) pseudo norm) to linear equations is NP hard and therefore, not practical in the intended applications due to the size of the data that is typically encountered there. Fortunately, it is now well established that under very mild conditions, the simple convex heuristic of minimizing the \( \ell_1 \)-norm 
\begin{equation}
\begin{cases}
\begin{aligned}
& \minimize && \norm{\repvec}{1} \\
& \sbjto && \samplevec = \linmap (\repvec) ,
\end{aligned}
\end{cases}
\end{equation}
instead of the \( \ell_0 \)-penalty, finds the sparsest solution almost always. Thus, the true signal can be recovered exactly by simply solving a convex optimization problem. Moreover, even if the linear measurements are noisy, recovery done via minimizing the \( \ell_1 \)-penalty is reasonably accurate.

Similar to compressed sensing is the problem of low rank matrix recovery / completion, where the objective is to reconstruct an entire matrix \( M' \) from only a few of its entries \( [M']_{ij} \) for \( (i,j) \in I \). Matrix recovery / completion problems arise regularly in recommender systems with case in point, the Netflix challenge. Since the unknown matrix is known to be of low rank, one expects that the true matrix is the solution to the following rank minimization problem.
\[
\begin{cases}
\begin{aligned}
& \minimize && \rank (M) \\
& \sbjto && [M']_{ij} = [M]_{ij} \quad \text{for } (i,j) \in I .
\end{aligned}
\end{cases}
\]
However, minimizing rank exactly, is extremely hard and impractical for most applications. It is now well established \cite{candes2009exact}, \cite{recht2010guaranteed}, \cite{chandrasekaran2012convex} that under mild conditions, the simple convex heuristic of minimizing the matrix nuclear norm \( \norm{\cdot}{*} \), recovers the true matrix.
\begin{equation}
\begin{cases}
\begin{aligned}
& \minimize && \norm{M'}{*} \\
& \sbjto && [M']_{ij} = [M]_{ij} \quad \text{for } (i,j) \in I .
\end{aligned}
\end{cases}
\end{equation}

Often, signals that are encountered in practice have low dimensional characteristics to them. In particular, they can be written as a linear combination of relatively few elements from some atomic set \( \mathcal{A} \). The low dimensional characteristic of the signal depends on this atomic set. For instance, in compressed sensing, since the signal is assumed to be sparse, the atomic set \( \mathcal{A} \) is considered to be the standard euclidean basis of appropriate dimension. Whereas, in the matrix recovery problem, since the unknown matrix is assumed to be of low rank, it can be written as the sum of a few rank-\( 1 \) matrices, and thus the atomic set \( \mathcal{A} \) is the set of all rank-\( 1 \) matrices of appropriate dimensions. So, given such a signal with the corresponding atomic set \( \mathcal{A} \), an LIP tries to finding a linear combination of a few elements of the atomic set \( \mathcal{A} \) that satisfy the given linear measurements of the signal. However, as seen in compressed sensing and matrix recovery problems, finding such a linear combination by simply searching the atomic set is impractical.

It is to be observed that the \( \ell_1 \) and the nuclear norms are the \emph{guage} functions corresponding to the convex hull of the standard euclidean basis (atomic set in compressed sensing) and the set of rank-\( 1 \) matrices (atomic set in matrix recovery problem) respectively. By minimizing such convex functions subject to the linear measurement constraints, guarantees have been obtained for fruitful signal recovery in compressed sensing and matrix recovery problems. Motivated by this observation, it was established in \cite{chandrasekaran2012convex} that for a generic LIP with a generic atomic set \( \mathcal{A} \), the analogous convex heuristic of minimizing the guage function corresponding to the set \( \convhull (\mathcal{A}) \) provides exact recovery under mild conditions. Thus, given linear measurements \( \samplevec = \linmap (\repvec') \) of a signal \( \repvec' \), an LIP seeks to solve
\begin{equation}
\label{eq:LIP}
\begin{cases}
\begin{aligned}
& \minimize && \cost (\repvec) \\
& \sbjto && \samplevec = \linmap (\repvec) ,
\end{aligned}
\end{cases}
\end{equation}
where \( \cost \) is a positively homogenous convex cost function such that \( \convhull (\mathcal{A}) = \{ \repvec : \cost (\repvec) \leq 1 \} \). If the observed linear measurements are noisy, a slightly modified problem is solved
\begin{equation}
\label{eq:robust-LIP}
\begin{cases}
\begin{aligned}
& \minimize && \cost (\repvec') \\
& \sbjto && \inpnorm{\samplevec - \linmap (\repvec')} \leq \error ,
\end{aligned}
\end{cases}
\end{equation}
where \( \error \geq 0 \) is a ``statistically   good'' bound on the measurement noise. 

If \( \repvec \) is a linear combination of only a ``few'' elements of \( \mathcal{A} \), true signal can be recovered from only the linear measurements by solving the LIP \eqref{eq:LIP}. Quantitative analysis on the number of measurements required to guarantee fruitful recovery, can be found in \cite{chandrasekaran2012convex} along with the constraints on the type of measurements that are suited for a given atomic set.

A better understanding of various LIPs has led to the inception of state of the art methods in medical imaging, recommender systems, image processing applications like denoising, super resolution and many more. A detailed exposition to a variety of applications where solving an LIP is central can be found in \cite{chandrasekaran2012convex}. Due to their importance, we would like to have simple, fast and efficient algorithms that solve these LIPs. One of the objective of the work carried out in this article is to provide the same. 

The other primary motivation for our work in this article comes from the related \emph{Dictionary Learning} problem. The objective in dictionary learning is to find a standard database of vectors called the \emph{dictionary} such that the given data samples \( (\samplevec_t)_t \) are expressed as linear combinations of the dictionary vectors. For a given dictionary, the linear combination corresponding to every sample is regarded as the representation of the respective sample under that particular dictionary. Depending on the application, a dictionary is learned so that the resulting representation of the data has some desirable features. One of the main feature to have in the representation is sparsity. With the recent success of sparsity based techniques in a bewildering range of topics in signal processing, the importance of sparsity in modern day data science can hardly be overstated. Due to the many benefits of sparse representation in applications such as compression, robustness, clustering etc., there is an ever increasing demand to learn good dictionaries that offer maximally sparse but also reasonably accurate representation of the data.  A brief overview on the relevance of the dictionary learning problem and methods used to learn a `good' dictionary are given in \cite{tosic2011dictionary}.

To this end, let \( \dictionary = \pmat{\dict{1}}{\dict{2}}{\dict{\dsize}} \) denote a dictionary of \( \dsize \) vectors, where \( \dsize \) is some positive integer. Let \( \repvec_t \) denote the representation of sample \( \samplevec_t \) for every \( t \in 1,2,\ldots, T \). Then the dictionary learning problem that we aim to solve is written as
\begin{equation}
\label{eq:DL-fixed-error}
\begin{cases}
\begin{aligned}
		& \minimize_{ (\repvec_t)_t, \; \dictionary } && \frac{1}{T} \sum\limits_{t = 1}^T \norm{\repvec_t}{1} \\
		& \sbjto				&& 
		\begin{cases}
		    \dictionary \in \unitdictionaryset , \\
			\norm{\samplevec_t - \dictionary \repvec_t }{2} \leq \error_t \quad \text{for every \( t = 1, 2, \ldots, T \)},
		\end{cases}
\end{aligned}
\end{cases}
\end{equation}
where \( \unitdictionaryset \) is some convex subset of \( \R{\dimension \times \dsize} \) and \( \error_t \) is a non-negative real number. In many image processing applications like denoising etc., \( \error_t \) corresponds to the bounds on the additive noise in the noisy data. Minimization of the \( \ell_1 \) penalty enforces sparsity in the representation vectors \( (\repvec_t)_t \). For different applications, the dictionary can be learned to optimize a generic cost function \( \cost (\cdot) \) that is task specific, instead of the \( \ell_1 \)-norm. It is to be noted that for a fixed dictionary \( \dictionary \), optimization over \( (\repvec_t)_t \) simply is to solve the LIP \eqref{eq:robust-LIP} for each \( t \).

Conventionally, a dictionary is learnt by solving the following optimization problem
\begin{equation}
	\label{eq:DL-l1-regularized}
	\minimize_{(\repvec_t)_t, \; \dictionary \; \in \; \unitdictionaryset } \quad \frac{1}{T} \sum\limits_{t = 1}^T  \Big( \norm{ f_t }{1} \; + \; \gamma \norm{x_t - \dictionary f_t}{2}^2  \Big)  
\end{equation}
where \( \gamma > 0 \) is the regularization parameter. It should be noted that the cost function in \eqref{eq:DL-l1-regularized} is a weighted cost of the sparsity inducing \( \ell_1 \)-penalty \( \norm{\repvec_t}{1} \) and the error term \( \norm{x_t - \dictionary f_t}{2}^2 \). The regularization parameter \( \gamma \) influences the tradeoff between the level of sparsity and the error. For a given value of \( \gamma \), this tradeoff is specific to a given distribution or data set. However, the precise relation between the value of regularization parameter \( \gamma \) and the tradeoff is not straightforward. Thus, apriori one doesn't know which value of the regularization parameter \( \lambda \) to choose for a given distribution or data set. It is a tuning parameter that has to be learned from the data. 

In various image processing applications like compressed sensing \cite{donoho2006compressed}, \cite{candes2008introduction}, inpainting, denoising \cite{elad2006image} etc., good estimates of \( (\error_t)_t \) to be used in \eqref{eq:DL-fixed-error} are known apriori. In contrast, if we were to learn the dictionary for the same applications but by solving \eqref{eq:DL-l1-regularized} instead, the appropriate value of the regularization parameter to be used is not known. This poses additional computational challenges. Furthermore, with a single parameter the problem formulation \eqref{eq:DL-l1-regularized} doesn't provide the level of customisability that exists in \eqref{eq:DL-fixed-error}. Therefore, learning dictionaries by solving \eqref{eq:DL-fixed-error} is more appealing in situations where good estimates of \( (\error_t)_t \) to be used are known beforehand.

One of main challenges in solving \eqref{eq:DL-fixed-error} is that there is no provable algorithm that computes an optimal dictionary for \eqref{eq:DL-fixed-error}. The existing techniques \cite{mairal2010online}, \cite{aharon2006k}, that learn a dictionary aim to solve \eqref{eq:DL-l1-regularized} by employing the technique of alternating the minimization over the variables \( (\repvec_t)_t \) and \( \dictionary \).
\begin{equation}
\begin{cases}
\begin{aligned}
	\label{eq:Bach-alternating-problems}
		 & \minimize_{(\repvec_t)_t} && \frac{1}{T} \sum\limits_{t = 1}^T  \Big( \norm{ f_t }{1} \; + \; \gamma \norm{x_t - \dictionary f_t}{2}^2 \Big) \; , \text{ and} \\
         & \minimize_{\dictionary \; \in \; \unitdictionaryset} && \frac{1}{T} \sum\limits_{t = 1}^T \norm{x_t - \dictionary f_t}{2}^2  .
\end{aligned}
\end{cases}
\end{equation}
It should be noted that, individually both the problems in \eqref{eq:Bach-alternating-problems} are convex problems. In particular, the optimization over the dictionaries is a QP and admits very easy and effective co-ordinate descent like algorithms. In contrast, we notice that such an alternating minimization technique \eqref{eq:Bach-alternating-problems} is completely ineffective in order to solve \eqref{eq:DL-fixed-error}. In particular, once the variables \( (\repvec_t)_t \) are fixed, there is no evident way to update the dictionary variable such that the resulting dictionary minimizes the cost. We observe that the objective function in \eqref{eq:DL-l1-regularized} depends directly on the dictionary variable \( \dictionary \), whereas, it affects the objective function of \eqref{eq:DL-fixed-error} indirectly. This is the main reason which makes it difficult to update the dictionary in any meaningful way when solving \eqref{eq:DL-fixed-error}. The equivalent reformulation of the LIP provided in this article, is crucial in tackling this issue. When learning a dictionary, we replace the LIPs of optimizing over variables \( (\repvec_t)_t \) with their respective equivalent reformulations provided in this article. Interestingly, by doing so, the dictionary variable is pushed to the cost function in a meaningful manner. Making use of this reformulation, we have provided dictionary learning methods in \cite{sheriff2019dictionary} that update the dictionary variable in \eqref{eq:DL-fixed-error} by performing an ascent-descent like algorithm. To the best of our knowledge, these are the first set of results that effectively solve the dictionary learning problem for situations where solving the formulation \eqref{eq:DL-fixed-error} is natural.

The article unfolds as follows: In Section \ref{section:Problem-statement-and-main-result} we formally introduce the LIP in a more generalised form and provide the main results including the equivalent convex-concave min-max reformulation. In Section \ref{section:convex-geometry}, we expose the duality by studying the underlying convex geometry of the LIP and provide proofs for all the results. We employ standard notations, and specific ones are explained as they appear.

\section{Formal problem statement and main results}
\label{section:Problem-statement-and-main-result}
Let \( \dimension \) be a positive integer, \( \hilbert \) be an \( \dimension \)-dimensional Hilbert space equipped with an innerproduct \( \inprod{\cdot}{\cdot} \) and its associated norm \( \inpnorm{\cdot} \). For every \( z \in \hilbert \) and \( r > 0 \), let \( \nbhood{z}{r} \define \{ y \in \hilbert : \inpnorm{\samplevec - y} < \error \} \) and let \( \closednbhood{z}{r} \define \{ y \in \hilbert : \inpnorm{\samplevec - y} \leq \error \} \). Let \( \cost : \R{\dsize} \longrightarrow [0, +\infty[ \) be a cost function such that it satisfies the following assumption.
\begin{assumption}
\label{assumption:cost-function}
    The cost function \( \cost : \R{\dsize} \longrightarrow [0, +\infty[ \) has the following properties
\begin{itemize}
        \item \emph{Positive Homogeneity} : There exists a positive real number \( \horder \) such that for every \( \alpha \geq 0 \) and \( \repvec \in \R{\dsize} \), we have \( \cost (\alpha \repvec ) = \alpha^{\horder} \cost (\repvec) \).
        
        \item \emph{Pseudo-Convexity} : The unit sublevel set \( \costatomset \define \{ \repvec \in \R{\dsize} : \cost (\repvec) \leq 1 \} \) is convex.
        
        \item \emph{Inf-Compactness} : The unit sublevel set \( \costatomset \) is compact.
\end{itemize}
\end{assumption}

Let \( \samplevec \in \hilbert \), non-negative real numbers \( \error \) and \( \regulizer \), and the linear map \( \linmap : \hilbert \longrightarrow \R{\dsize} \) be given. We consider the following general formulation of the linear inverse problem
\begin{equation} 
	\label{eq:coding-problem}
	\begin{cases}
	  \begin{aligned}
		& \minimize_{( \mathsf{c},\; \repvec ) \; \in \; \R{} \times \R{\dsize}}  && \quad \mathsf{c}^{\horder} \\
		& \sbjto							  &&
		\begin{cases}
			\big( \cost (\repvec) \big)^{1/\horder} \leq \mathsf{c} \\
            \inpnorm{ \samplevec - \linmap (\repvec) } \leq \error + \regulizer \mathsf{c} .
		\end{cases}
	\end{aligned}
	\end{cases}
\end{equation}
When \( \regulizer = 0 \), we see that the feasible collection of \( \repvec \) is independent from the variable \( \mathsf{c} \). As a consequence we see that for every feasible \( \repvec \in \R{\dsize} \), the minimization over the variable \( \mathsf{c} \) is achieved for \( \mathsf{c} = \cost (\repvec) \). Thus the linear inverse problem \eqref{eq:coding-problem} reduces to the following more familiar formulation.
\begin{equation} 
	\label{eq:coding-problem-absolute-error}
	\begin{cases}
	 \begin{aligned}
		& \minimize_{\repvec \; \in \; \R{\dsize}} && \cost ( \repvec ) \\
		& \sbjto && \inpnorm{ \samplevec - \linmap (\repvec) } \leq \error .
	\end{aligned}
	\end{cases}
\end{equation}
The non-negative real number \( \regulizer \) acts as a regularization parameter. If \( \regulizer > 0 \), by considering \( \mathsf{c} > \big( \inpnorm{\samplevec} / \regulizer \big)^{\horder} \) and \( \repvec = 0 \), we see that the linear inverse problem \eqref{eq:coding-problem} is always feasible. On the contrary, if \( \regulizer = 0 \), it is immediately apparent that \eqref{eq:coding-problem} is feasible if and only if \( \closednbhood{\samplevec}{\error} \cap \image (\linmap) \neq \emptyset \).

It might be surprising at first to see the rather unusual formulation \eqref{eq:coding-problem} of the linear inverse problem. Our formulation makes way for the possibility of \( \regulizer \) to take positive values also. BY considering a positive value for the regularization parameter \( \regulizer \), we obtain several advantages:
\begin{itemize}[leftmargin = *]
    \item A positive value of \( \regulizer \) amounts to having the effect of regularization in the problem. Thus, one can harvest all the advantages that come from regularization like numerical stability in algorithms, well conditioning etc.
    
    \item Whenever \( \regulizer > 0 \), the LIP \eqref{eq:coding-problem} is always strictly feasible. This is a crucial feature in the initial stages of dictionary learning, in particular, when the data lies in a subspace of lower dimension \( m \), such that \( m, \dsize \ll \dimension \), where \( \dsize \) is the number of dictionary vectors.
    
    \item Having a positive value of \( \regulizer \) eliminates the pathological cases that arise in dictionary learning and provides guarantees for convergence of dictionary learning algorithms. Moreover, it leads to useful fixed point characterization of the optimal dictionary, which in turn lead to simple online dictionary update algorithms.
\end{itemize}

\subsection{Main results and discussion}
We observe that the mapping \( \repvec \longmapsto (\cost (\repvec))^{1/\horder} \) is an inf-compact, convex and positively homogeneous of order \( 1 \). Therefore, it is immediate that the constraints of the LIP \eqref{eq:coding-problem} are convex. Furthermore, the objective function is also convex whenever \( \horder \geq 1 \). Thus, it is apparent that the LIP \eqref{eq:coding-problem} is a convex problem when \( \horder \geq 1 \). When \( \horder < 1 \), we highlight that \( [0, +\infty[ \ni (\cdot) \longmapsto (\cdot)^{\horder} \) is an increasing function, and therefore, minimizing \( \mathsf{c}^{\horder} \) is equivalent to minimizing \( \mathsf{c} \). Thus, the LIP \eqref{eq:coding-problem} has an underlying convex problem (except that the objective function is a non-convex power).

We emphasise that whenever the optimization problem \eqref{eq:coding-problem} is feasible, the feasible set is closed and the cost function is continuous and \emph{coercive}.\footnote{Recall that a continuous function \( \cost \) defined over an unbounded set \( U \) is said to be coercive in the context of an optimization problem, if : \( \lim\limits_{\inpnorm{u} \to \infty } \cost (u) = +\infty \; (- \infty), \) in the context of minimization (maximization) of \( \cost \) and the limit is considered from within the set \( U \).
}
Therefore, from the Weierstrass theorem \cite[Theorem 4.16]{rudin1964principles} we conclude that whenever \eqref{eq:coding-problem} is feasible, it admits an optimal solution. To this end, let
\begin{equation}
\label{eq:definition-of-encoding-cost-and-codes}
\big( (\samplecost)^{\frac{1}{\horder}} , \codes \big) \define
\begin{cases}
	  \begin{aligned}
		& \argmin_{( \mathsf{c},\; \repvec ) \; \in \; \R{} \times \R{\dsize}}  && \quad \mathsf{c}^{\horder} \\
		& \sbjto							  &&
		\begin{cases}
			\big( \cost (\repvec) \big)^{1/\horder} \leq \mathsf{c} \\
            \inpnorm{ \samplevec - \linmap (\repvec) } \leq \error + \regulizer \mathsf{c} .
		\end{cases}
	\end{aligned}
	\end{cases}
\end{equation}
Note that \( \samplecost \) is also the optimal value achieved in \eqref{eq:coding-problem}. In view of this, we slightly abuse the definition \eqref{eq:definition-of-encoding-cost-and-codes}, and follow the convention that if \eqref{eq:coding-problem} is infeasible, then \( \samplecost \define +\infty \) and \( \codes \define \emptyset \).

The optimization problem \eqref{eq:coding-problem} depending on the parameters \( \samplevec, \linmap, \error , \regulizer \) could potentially have multiple solutions. However, in signal recovery from ill posed linear inverse problems, if there are sufficiently many linear measurements, and of correct type, the LIP \eqref{eq:coding-problem} admits a unique solution and the set \( \codes \) is then a singleton containing the true signal to be recovered. In other situations like sparse encoding, the LIP \eqref{eq:coding-problem} could have multiple solutions, and if it does, we see that \( \codes \) is a convex set.

\begin{definition}
\label{def:representability}
Let \( \linmap : \R{\dsize} \longrightarrow \hilbert \) be a linear map, and let \( \error , \regulizer \geq 0 \). A vector \( \samplevec \in \hilbert \) is said to be \( \Depsdelta \)-feasible if \( \samplecost < +\infty \).
\end{definition}
\begin{remark}
We see that \( \samplevec \in \hilbert \) is \( \Depsdelta \)-feasible if and only if at least one of the following holds:
\begin{itemize}
    \item \( \regulizer > 0 \),
    \item \( \closednbhood{\samplevec}{\error} \cap \image(\linmap) \neq \emptyset \).
\end{itemize}
\end{remark}

\begin{definition}
\label{definition:atomball}
For the linear map \( \linmap : \hilbert \longrightarrow \R{\dsize} \), non-negative real number \( \regulizer \) and the cost function \( \cost : \hilbert \longrightarrow [0, +\infty [ \) satisfying Assumption \ref{assumption:cost-function}, let us define
\begin{equation}
\label{eq:atomic-ball-definition}
\atomball \define \{ z \in \hilbert : \text{there exists \( \repvec \in \costatomset  \) satisfying \( \inpnorm{z - \linmap (\repvec)} \leq \regulizer \) } \} ,
\end{equation}
where \( \costatomset \) is the unit sub level set of the cost function \( \cost \).
\end{definition}
By denoting \( S' \coloneqq \{ \linmap (\repvec) : \repvec \in \costatomset \} \), it is clear that \( S' \) is the image of the compact and convex set \( \costatomset \) under the linear map \( \linmap \), and is therefore, compact and convex.\footnote{Considering, for instance, \( \cost (\cdot) = \norm{\cdot}{1} \) and the linear map \( \linmap \) given by the matrix \( \dictionary = \pmat{\dict{1}}{\dict{2}}{\dict{\dsize}} \in \R{\dimension \times \dsize} \), we see that \( \costatomset \) is the \( \ell_1 \)-closed ball in \( \R{\dsize} \) and \( S' = \convhull (\pm \dict{i})_{i = 1}^{\dsize} \).} Moreover, since the set \( \atomball \) is the image of the linear map \( S' \times \closednbhood{0}{\regulizer} \ni (z', y) \longmapsto z' + y \), \( \atomball \) is also compact and convex. Furthermore, for every \( \scaling \geq 0 \) the set \( \atomset{\scaling} \define \scaling \cdot \atomball \), obtained by linearly scaling \( \atomball \) by an amount of \( \scaling \), is also compact and convex.

The \emph{guage function} \( \norm{\cdot}{\linmap} : \hilbert \longrightarrow [0 , +\infty [ \) corresponding to the set \( \atomball \)is given by
\begin{equation}
\label{eq:guage-function-definition}
\norm{z}{\linmap} \define \min\big\{\scaling \geq 0 : z \in \atomset{\scaling} \big\} .
\end{equation}
When \( \regulizer > 0 \), we know that \( \atomball \) has non-empty interior, therefore, \( \norm{z}{\linmap} < +\infty \) for every \( z \in \hilbert \). Similarly when \( \regulizer = 0 \), \( \norm{z}{\linmap} < +\infty \) if and only if \( z \in \image (\linmap) \). Moreover, due to the set \( \atomset{\scaling} \) being compact for every \( \scaling \geq 0 \), the minimization in the definition of the guage function is always achieved. In other words, for every \( z \in \hilbert \) such that \( \norm{z}{\linmap} < +\infty \), we have \( z \in \atomset{\norm{z}{\linmap}} \).

The underlying convexity of the linear inverse problem \eqref{eq:coding-problem} gives rise to an interplay of the convex bodies \( \closednbhood{\samplevec}{\error} \) and \( \atomball \). As a result, we obtain the relation between the optimal cost \( \samplecost \), the guage function \( \norm{\cdot}{\linmap} \) and the set \( \closednbhood{\samplevec}{\error} \).
\begin{lemma}
\label{lemma:coding-problem-guage-function-equivalent}
Consider the LIP \eqref{eq:coding-problem} for the linear map \( \linmap \), cost function \( \cost \), non-negative real numbers \( \error, \regulizer \) and \( \samplevec \in \hilbert \). The optimal cost \( \samplecost \) of the LIP \eqref{eq:coding-problem} and the guage function \( \norm{\cdot}{\linmap} \) satisfy
\begin{equation}
\label{eq:coding-problem-guage-function-equivalent}
(\samplecost)^{1/\horder} = \min_{y \; \in \;  \closednbhood{\samplevec}{\error}} \; \norm{y}{\linmap} .
\end{equation}
\end{lemma}

\subsubsection{Duality}
The guage function \( \norm{\cdot}{\linmap} \) gives rise to its corresponding dual function \( \dualnorm{\cdot} : \hilbert \longrightarrow [0 , +\infty[ \) defined by:
\begin{equation}
\label{eq:dual-norm-definition}
\dualnorm{\separator} \define \sup\limits_{\norm{z}{\linmap} \leq 1} \inprod{\separator}{z} \; = \sup\limits_{z \in \atomball} \; \inprod{\separator}{z} . 
\end{equation}
Let \( \separator , y \in \hilbert \), we recall that \( y \in \atomset{\norm{y}{\linmap}} \), and consequently, we have the following Holder like inequality
\[
\inprod{\separator}{y} \; \leq \sup_{z \in \atomset{\norm{y}{\linmap}}} \inprod{\separator}{z}\; = \; \norm{y}{\linmap} \sup_{z \in \atomball} \inprod{\separator}{z} \; = \; \norm{y}{\linmap} \dualnorm{\separator} .
\]
This gives rise to strong duality between the guage function \( \norm{\cdot}{\linmap} \) and its associated dual function \( \dualnorm{\cdot} \) in the following way
\begin{equation}
\label{eq:strong-duality-of-guage-function}
\norm{y}{\linmap} =
\begin{cases}
\begin{aligned}
& \sup_{\separator} && \inprod{\separator}{y} \\
& \sbjto && \dualnorm{\separator} \leq 1 .
\end{aligned}
\end{cases} 
\end{equation}
By replacing the guage function \( \norm{\cdot}{\linmap} \) in \eqref{eq:coding-problem-guage-function-equivalent} with its equivalent sup formulation provided in \eqref{eq:strong-duality-of-guage-function}, we obtain the Convex Dual of the LIP \eqref{eq:coding-problem}. First, we will define the set \( \separatorset \) which is the collection of optimal dual variables.
\begin{definition}
\label{def:optimal-separator-set}
Let the linear map \( \linmap \) and \( \error, \regulizer \geq 0 \) and a cost function \( \cost \) satisfying Assumption \ref{assumption:cost-function} be given. Then for every \( \samplevec \in \hilbert \) that is \( \Depsdelta \)-feasible, let \( \separatorset \subset \hilbert \) denote the collection of points \( \separator \in \hilbert \setminus \closednbhood{0}{\error} \) that satisfy the following two conditions simultaneously:
\begin{itemize}
     \item \( \dualnorm{\separator} = \ 1 \), and
    \item \( \inprod{\separator}{\samplevec} - \error \inpnorm{\separator} \ = \ (\samplecost)^{1/\horder} \).
\end{itemize}
\end{definition}

\begin{theorem}
\label{theorem:coding-problem-equivalent-sup-problem}
Let the linear map \( \linmap : \R{\dsize} \longrightarrow \hilbert \), real numbers \( \error, \regulizer \geq 0 \) and \( \samplevec \in \hilbert \) be given. Consider the linear inverse problem \eqref{eq:coding-problem} and its convex dual problem:
\begin{equation}
\label{eq:coding-problem-equivalent-sup-problem}
\begin{cases}
\begin{aligned}
& \sup_{\separator} && \inprod{\separator}{\samplevec} - \error \inpnorm{\separator}  \\
&\sbjto && \dualnorm{\separator} \leq 1 .
\end{aligned}
\end{cases}
\end{equation}
\begin{enumerate}[label = \rm{(\roman*)}, leftmargin=*]
\item Strong Duality: The supremum value in \eqref{eq:coding-problem-equivalent-sup-problem} is finite if and only if \( \samplevec \) is \( \Depsdelta \)-feasible, and is equal to the optimal cost \( ( \samplecost )^{\frac{1}{\horder}} \).

\item Existence and description of an optimal solution to \eqref{eq:coding-problem-equivalent-sup-problem}.
\begin{enumerate}[label = \rm{(\alph*)}, leftmargin=*]
\item Irrespective of the value of \( \regulizer \), for any \( \error \geq 0 \) if \( \inpnorm{\samplevec} \leq \error \), then \( \separator\opt = 0 \) is an optimal solution.

\item Whenever \( \inpnorm{\samplevec} > \error \), the optimization problem \eqref{eq:coding-problem-equivalent-sup-problem} admits an optimal solution if and only if the set \( \separatorset \) defined in \ref{def:optimal-separator-set} is non-empty and \( \separator\opt \) is a solution if and only if \( \separator\opt \in \separatorset \). As a result, the supremum is indeed a maximum and it is achieved at \( \separator\opt \).

\item Whenever \( \inpnorm{\samplevec} > \error \) and the set \( \separatorset \) is empty, the optimization problem \eqref{eq:coding-problem-equivalent-sup-problem} does not admit an optimal solution even though the value of the supremum is finite.
\end{enumerate}
\end{enumerate}
\end{theorem}

\begin{remark}
We provide a complete description of the set \( \separatorset \) in Proposition \ref{proposition:description-of-the-separator-set}. It turns out that the optimal solution to the dual problem \eqref{eq:coding-problem-equivalent-sup-problem} can be entirely characterization in terms of the optimal solution \( \big( \samplecost , \codes \big) \) to the LIP \eqref{eq:coding-problem} itself. Therefore, if we only have access to a black box that produces an optimal solution to the LIP \eqref{eq:coding-problem}, a corresponding dual optimal solution can be easily computed from the solutions to the LIP itself, and we need not solve the dual problem again separately. This is very advantageous in dictionary learning, where the optimal value of dual variable is required to compute a better dictionary.
\end{remark}

\begin{remark}
We look ahead at Proposition \ref{proposition:description-of-the-separator-set} and see that the dual problem does not admit any optimal solution only when \( \regulizer = 0 \) and \( \nbhood{\samplevec}{\error} \cap \image (\linmap) = \emptyset \). Interestingly, in that case, we also observe that the corresponding primal problem is not strictly feasible.
\end{remark}

\subsubsection{Equivalent min-max formulation of the LIP}
Whenever \( \inpnorm{\samplevec} \leq \error \), we immediately see that the pair \( \R{}_+ \times \R{\dsize} \ni (\mathsf{c}_{\samplevec}, \repvec_{\samplevec} ) \coloneqq (0,0) \) is feasible for \eqref{eq:coding-problem}. Moreover, since \( \cost (\repvec) > 0 \) for every \( \repvec \neq 0 \) due to inf-compactness and positive homogeneity, we conclude that:
\begin{equation}
\label{eq:zero-approximate-sample}
\samplecost = 0 \quad \text{if and only if } \inpnorm{\samplevec} \leq \error .
\end{equation}
Therefore, the case: \( \inpnorm{\samplevec} \leq \error \) is uninteresting and inconsequential. In fact, in dictionary learning, since every such sample can be effectively represented by the zero vector, irrespective of the dictionary. The average cost of representation then depends only on the samples that satisfy \( \inpnorm{\samplevec} > \error \). Therefore, for the convex-concave min-max formulation of the LIP \eqref{eq:coding-problem}, we consider only the case when \( \inpnorm{\samplevec} > \error \).

\begin{theorem}
\label{} 
Let the linear map \( \linmap : \R{\dsize} \longrightarrow \hilbert \) , real numbers \( \error , \regulizer \geq 0 \),  \( q \in ]0 , 1[ \), \( r > 0 \) and \( \samplevec \in \hilbert \setminus \closednbhood{\samplevec}{\error} \) be given. Consider the linear inverse problem \eqref{eq:coding-problem} and the following inf-sup problem:
\begin{equation}
\label{eq:coding-problem-primal-dual}
\begin{cases}
\begin{aligned}
& \min_{\hvar \; \in \; \costatomset} \; \sup_{\separator } \ \ && r \Big( \inprod{\separator}{\samplevec} -  \error \inpnorm{\separator} \Big)^q - \Big( \regulizer \inpnorm{\separator} + \inprod{\separator}{\linmap (\hvar)} \Big) \\
& \sbjto &&  \inprod{\separator}{\samplevec} - \error \inpnorm{\separator} > 0 .
\end{aligned}
\end{cases}
\end{equation}
The following assertions hold with regards to the coding problem \eqref{eq:coding-problem} and its equivalent \eqref{eq:coding-problem-primal-dual}. 
\begin{enumerate}[leftmargin = *, label = \rm{(\roman*)}]
\item The optimal value of \eqref{eq:coding-problem-primal-dual} is equal to: \( \; s(r , q) \; ( \samplecost )^\frac{q}{\horder (1 - q)} \), and therefore finite if and only if \( \samplevec \) is \( \Depsdelta \)-feasible.\footnote{ \( s(r, q) \define \Big( (1 - q) (q^q r )^{\frac{1}{1 - q}} \Big) \; \) }

\item Existence and description of an optimal solution to \eqref{eq:coding-problem-primal-dual}.
\begin{enumerate}[leftmargin = *, label = \rm{(\alph*)}]
\item The minimization over variables \( \hvar \) in \eqref{eq:coding-problem-primal-dual} is achieved, and 
\begin{equation*}
\hvar\opt \in \;
\argmin\limits_{\hvar \; \in \; \costatomset}
\begin{cases}
\begin{aligned}
& \sup_{\separator } \ \ &&  r \Big( \inprod{\separator}{\samplevec} -  \error \inpnorm{\separator} \Big)^q - \Big( \regulizer \inpnorm{\separator} + \inprod{\separator}{\linmap (\hvar)} \Big) \\
& \sbjto &&  \inprod{\separator}{\samplevec} - \error \inpnorm{\separator} > 0 ,
\end{aligned}
\end{cases}    
\end{equation*}
if and only if \( \hvar\opt \in \frac{1}{(\samplecost)^{1/\horder}} \cdot \codes \).

\item In addition, the inf-sup problem \eqref{eq:coding-problem-primal-dual} admits a saddle point solution if and only if the set \( \separatorset \) is non-empty, then a pair \( (\hvar\opt , \separator\opt) \in \costatomset \times \hilbert \) is a saddle point solution to \eqref{eq:coding-problem-primal-dual} if and only if 
\[
\begin{aligned}
\hvar\opt & \in \frac{1}{(\samplecost)^{1/\horder}} \cdot \codes \ \text{and} \\
\separator\opt & \in ( r q )^{\frac{1}{1 - q}} \big( \samplecost \big)^{\frac{q}{\horder (1 - q)}} \cdot \separatorset .
\end{aligned}
\]
\end{enumerate}
 \end{enumerate}
\end{theorem}

\begin{corollary}
\label{corollary:encoding-cost-min-max-problem}
By considering \( r = r (\horder) \define (1 + \horder) \horder^{\frac{ - \horder}{1 + \horder}} \) and \( q = q (\horder) \define \frac{\horder}{1 + \horder} \), we get
\[
\samplecost = 
\begin{cases}
\begin{aligned}
& \min_{\hvar \; \in \; \costatomset} \; \sup_{\separator } \ \ && r (\horder) \Big( \inprod{\separator}{\samplevec} -  \error \inpnorm{\separator} \Big)^{q (\horder)} - \Big( \regulizer \inpnorm{\separator} + \inprod{\separator}{\linmap (\hvar)} \Big) \\
& \sbjto &&  \inprod{\separator}{\samplevec} - \error \inpnorm{\separator} > 0 .
\end{aligned}
\end{cases}
\]
In particular, if the cost function \( \cost (\cdot) \) is positively homogeneous of order \( 1 \) (like any norm), we have
\[
\samplecost = 
\begin{cases}
\begin{aligned}
& \min_{\hvar \; \in \; \costatomset} \; \sup_{\separator } \ \ && 2 \sqrt{ \inprod{\separator}{\samplevec} -  \error \inpnorm{\separator} } - \Big( \regulizer \inpnorm{\separator} + \inprod{\separator}{\linmap (\hvar)} \Big) \\
& \sbjto &&  \inprod{\separator}{\samplevec} - \error \inpnorm{\separator} > 0 .
\end{aligned}
\end{cases}
\]
\end{corollary}

\begin{remark}
The objective function in the min-max problem \eqref{eq:coding-problem-primal-dual} is convex w.r.t. to minimization variable \( \hvar \) and concave w.r.t. the maximization variable \( \separator \).
\end{remark}

\begin{remark}
In several scenarios like the problem of non-negative matrix factorization \cite{lee2001algorithms}, \cite{lee1999learning}, \cite{hoyer2004non}, one has to solve \eqref{eq:coding-problem} with the additional constraint that \( \repvec \in Q \), where \( Q \subset \R{\dsize} \) is a convex cone. From similar analysis provided in this article, it can be easily verified that the resulting equivalent min-max formulation analogous to \eqref{eq:coding-problem-primal-dual} is
\[
\begin{cases}
\begin{aligned}
& \min_{\hvar \; \in \; \costatomset \cap Q} \; \sup_{\separator } \ \ && r \Big( \inprod{\separator}{\samplevec} -  \error \inpnorm{\separator} \Big)^q - \Big( \regulizer \inpnorm{\separator} + \inprod{\separator}{\linmap (\hvar)} \Big) \\
& \sbjto &&  \inprod{\separator}{\samplevec} - \error \inpnorm{\separator} > 0 .
\end{aligned}
\end{cases}
\]
It should be noted that the definition of the set \( \atomball \) then changes to
\[
\atomball \define \{ z \in \hilbert : \text{ there exists } \repvec \in \costatomset \cap Q \text{ satisfying } \inpnorm{\samplevec - \linmap (\repvec)} \leq \regulizer \} ,
\]
and the quantities \( \norm{\cdot}{\linmap} \), \( \dualnorm{\cdot} \) and \( \separatorset \) are accordingly defined w.r.t. the appropriate definition of the set \( \atomball \).
\end{remark}

\begin{remark}
When the order of homogeneity of the cost \( \cost \) is \( 1 \), the min-max problem \eqref{eq:coding-problem-primal-dual} can be alternatively written by eliminating the constraint \( \inprod{\separator}{\samplevec} - \error \inpnorm{\separator} > 0 \) in the following way
\[
\min_{\hvar \; \in \; \costatomset} \inf_{\dualvar > 0} \; \sup_{\separator } \quad \dualvar + \frac{1}{\dualvar} \Big( \inprod{\separator}{\samplevec} -  \error \inpnorm{\separator} \Big) - \Big( \regulizer \inpnorm{\separator} + \inprod{\separator}{\linmap (\hvar)} \Big) .
\]
Even though this formulation might suffice for solving a variety of practical LIPs, the exact form of \eqref{eq:coding-problem-primal-dual} is necessary and inevitable for solving the dictionary learning problem.
\end{remark}

\begin{remark}
When the error constraint \( \inpnorm{\samplevec - \linmap (\repvec)} \leq \error + \regulizer \mathsf{C} \) is measured using a generic norm \( \inpnorm{\cdot} \), the min-max problem is written using the corresponding dual norm \( \inpnorm{\cdot}' \) as
\[
\begin{cases}
\begin{aligned}
& \min_{\hvar \; \in \; \costatomset} \; \sup_{\separator } \ \ && r \Big( \inprod{\separator}{\samplevec} - \error \inpnorm{\separator}' \Big)^q - \Big( \regulizer \inpnorm{\separator}' + \inprod{\separator}{\linmap (\hvar)} \Big) \\
& \sbjto &&  \inprod{\separator}{\samplevec} - \error \inpnorm{\separator} > 0 .
\end{aligned}
\end{cases}
\]
Furthermore,if \( \regulizer = 0 \) and we replace the error constraint with \( \linmap (\repvec) \in B_{\samplevec} \), where \( B_{\samplevec} \subset \hilbert \setminus \{ 0 \} \) is some compact convex subset. The min-max problem is written as
\[
\begin{cases}
\begin{aligned}
& \min_{\hvar \; \in \; \costatomset} \; \sup_{\separator } \ \ && r \Big( \min_{y \in B_{\samplevec}} \inprod{\separator}{y} \Big)^q - \inprod{\separator}{\linmap (\hvar)} \\
& \sbjto &&  0 < \min_{y \in B_{\samplevec}} \inprod{\separator}{y} .
\end{aligned}
\end{cases}
\]
\end{remark}

\subsection{Signal recovery from linear measurements}
In the case of signal recovery from its linear measurements, it is assumed that the true signal is a linear combination of only a few elements of some atomic set \( \mathcal{A} \). The signal is recovered by solving the linear inverse problem \eqref{eq:coding-problem}, by considering the cost function \( \cost \) such that \( \costatomset = \convhull (\mathcal{A}) \) and \( \horder = 1 \). Conditions on minimum number and type of measurements of measurements for fruitful recovery, have been extensively discussed in \cite{chandrasekaran2012convex}. If those conditions are satisfied, then the set \( \codes \) is a singleton and contains the signal to be recovered.

\subsubsection{Basis Pursuit Denoising}
An example of the linear inverse problem which is of practical relevance is the classical \emph{Basis Pursuit Denoising} problem \cite{elad2006image}, \cite{candes2008introduction}, that arises in various scenarios of compressed sensing and image processing.
\begin{equation}
\label{eq:basis-pursuit-denoising}
\begin{cases}
\begin{aligned}
		& \minimize_{\repvec \; \in \; \R{\dsize}} && \norm{\repvec}{1} \\
		& \sbjto && \norm{ \samplevec - \linmap (\repvec) }{2} \leq \error .
\end{aligned}
\end{cases}
\end{equation}
The corresponding dual problem is
\begin{equation}
\begin{cases}
\begin{aligned}
    & \sup_{\separator}  && \separator\transp \samplevec - \error \norm{\separator}{2} \\
		& \sbjto							  &&
		 \norm{\linmap\transp \separator}{\infty} \; \leq 1 \; ,
\end{aligned}
\end{cases}
\end{equation}
and the min-max reformulation is
\begin{equation}
\label{eq:basis-pursuit-denoising-min-max-problem}
\begin{cases}
\begin{aligned}
& \min_{ \norm{\hvar}{1} \leq 1} \; \sup_{\separator } \ \ && 2 \sqrt{\separator\transp \samplevec - \error \norm{\separator}{2}} - \Big( \regulizer \norm{\separator}{2} + \separator\transp \linmap (\hvar) \Big) \\
& \sbjto &&  \separator\transp \samplevec - \error \norm{\separator}{2} > 0 .
\end{aligned}
\end{cases}
\end{equation}

\subsection{The dictionary learning problem}
The setup is that every vector \( \samplevec \in \hilbert \) is \emph{encoded} as a vector \( \repvec ( \samplevec) \) in \( \R{\dsize} \) via the \emph{encoder} map \( \repvec : \hilbert \longrightarrow \R{\dsize} \). We shall refer to \( \repvec (\samplevec) \) as the \emph{representation} of \( \samplevec \) under the encoder \( \repvec \). The \emph{reconstruction} of the encoded samples from the representation \( \repvec (\samplevec) \) is done by taking the linear combination \( \sum\limits_{i = 1}^{\dsize} \repvec_i (\samplevec) \dict{i} \) with some standard collection of vectors \( \dictionary \coloneqq \pmat{\dict{1}}{\dict{2}}{\dict{\dsize}} \) referred to as the \emph{dictionary}. Since the reconstruction has to be a good representative of the true vector \( \samplevec \), we constraint the error \( \inpnorm{\samplevec - \dictionary \repvec(\samplevec)} \) to be small.

Given a dictionary \( \dictionary \), we encode every vector \( \samplevec \) by solving the LIP \eqref{eq:coding-problem} with an appropriate cost function \( \cost \). This cost function determines the desirable characteristics in the representation. In other words, the optimal encoder map \( \repvec_{\dictionary} : \hilbert \longrightarrow \R{\dsize} \) corresponding to the dictionary \( \dictionary \) is such that \( \repvec_{\dictionary} (\samplevec) \in \encodermap{\dictionary}{\samplevec}{\error}{\regulizer} \) for every \( \samplevec \). Our objective is to find dictionaries such that the corresponding encoder map \( \repvec_{\dictionary} \) has desirable features like sparsity, robustness with respect to loss of coefficients etc., in the representation. We refer to the task of finding such a dictionary as the \emph{dictionary learning problem}. 

Formally, let \( \PP \) be a distribution on \( \hilbert \) and \( \rv \) be a \( \PP \) distributed random variable. Let \( \cost : \R{\dsize} \longrightarrow \R{}_+ \) be a given cost function that satisfies Assumption \ref{assumption:cost-function}, \( \error : \hilbert \longrightarrow [0, +\infty[ \) be a given error threshold function and \( \regulizer \) be a non-negative real number. Given a dictionary \( \dictionary \), since the random variable \( \rv \) is encoded as \( \repvec_{\dictionary} (\rv) \in \encodermap{\dictionary}{\rv}{\error (\rv)}{\regulizer} \), we consider the cost incurred to encode to be \( \encodedcost{\dictionary}{\rv}{\error (\rv)}{\regulizer} \). Our objective is to find a dictionary that facilitates optimal encoding of the data, which are the samples drawn from \( \PP \). Therefore, we consider the following dictionary learning problem :
\begin{equation}
\label{eq:dictionary-learning-problem}
        \minimize_{\dictionary \; \in \; \unitdictionaryset} \quad \EE_{\PP} \big[ \encodedcost{\dictionary}{\rv}{\error (\rv)}{\regulizer} \big] ,
\end{equation}
where \( \unitdictionaryset \subset \R{\dimension \times \dsize} \) is some known compact convex subset. 

For a large integer \( \horizon \), let \( \data \coloneqq (\sample{t})_{t = 1}^{\horizon} \) be a collection of samples drawn from the distribution \( \PP \). Let us consider the dictionary learning problem for the sampled data, given by:
\begin{equation}
\label{eq:dictionary-learning-samples}
\minimize_{\dictionary \; \in \; \unitdictionaryset} \ \frac{1}{\horizon} \summ{t = 1}{\horizon} \encodedcost{\dictionary}{\sample{t}}{\error (\sample{t})}{\regulizer} \ .
\end{equation}
For the special case of \( \regulizer = 0 \), the dictionary learning problem \eqref{eq:dictionary-learning-samples} can be restated using the definition of the encoding cost \( \encodedcost{\dictionary}{\sample{t}}{\error (\sample{t})}{\regulizer} \)  in the more conventional form as:
\begin{equation}
\label{eq:DL-conventional}
\begin{cases}
\begin{aligned}
        &\minimize_{\dictionary , \; (\repvec_t)_t}  && \frac{1}{\horizon} \summ{t = 1}{\horizon} \cost (\repvec_t) \\
	    & \sbjto			&&
	    \begin{cases}
	       \dictionary \in \unitdictionaryset , \\
	       \repvec_t \in \R{\dsize} , \\
	       \inpnorm{ \sample{t} - \dictionary \repvec_t } \leq \error (\sample{t}) \text{ for all } t = 1,2,\ldots, \horizon .
	    \end{cases}
	\end{aligned}
\end{cases}
\end{equation}

The dependence of the encoding cost \( \encodedcost{\dictionary}{\samplevec}{\error}{\regulizer} \) on the dictionary variable \( \dictionary \) is not immediately evident. Therefore, it is replaced in the dictionary learning problem \eqref{eq:dictionary-learning-samples} with the min-max problem provided in Corollary \ref{corollary:encoding-cost-min-max-problem} to obtain
\[
\min_{\dictionary \in \unitdictionaryset} \ \min_{ (\hvar_t)_t \subset \costatomset }
\begin{cases}
\begin{aligned}
& \sup_{(\separator_t)_t} && \frac{1}{T} \summ{t = 1}{T} \Big( r(\horder) \big( \inprod{\separator_t}{\samplevec_t} -  \error \inpnorm{\separator_t} \big)^{\frac{\horder}{1 + \horder}} - \big( \regulizer \inpnorm{\separator_t} + \inprod{\separator_t}{\dictionary \hvar_t} \big) \Big) \\
& \text{ s.t.} && \inprod{\separator_t}{\samplevec_t} - \error \inpnorm{\separator_t} > 0 ,
\end{aligned}
\end{cases}
\]
where \( r(\horder) = (1 + \horder) \horder^{\frac{ - \horder}{1 + \horder}} \). The dictionary learning problem \eqref{eq:dictionary-learning-samples} is then solved by alternating the optimization over \( \dictionary \) and \( (\hvar_t)_t \) keeping the other one fixed. It is to be noted that each of these optimization problem is a min-max problem in variables \( (\dictionary , (\separator_t)_t ) \) and \( ( (\hvar_t)_t , (\separator_t)_t) \) respectively. In particular, for a given sequence \( (\hvar_t)_t \subset \costatomset \), the dictionary is updated by solving the following min-max problem
\begin{equation*}
\begin{cases}
\begin{aligned}
& \min_{\dictionary \; \in \; \unitdictionaryset} \ \sup_{(\separator_t)_t} && \frac{1}{T} \summ{t = 1}{T} \Big( r(\horder) \big( \inprod{\separator_t}{\samplevec_t} -  \error \inpnorm{\separator_t} \big)^{\frac{\horder}{1 + \horder}} - \big( \regulizer \inpnorm{\separator_t} + \inprod{\separator_t}{\dictionary \hvar_t} \big) \Big) \\
& \sbjto && \inprod{\separator_t}{\samplevec_t} - \error \inpnorm{\separator_t} > 0 .
\end{aligned}
\end{cases}
\end{equation*}
It is shown in \cite{sheriff2019dictionary} that if \( \regulizer > 0 \) the above min-max problem always admits a saddle point solution. Moreover, we observe the objective function of the min-max problem is linear w.r.t. the dictionary variable \( \dictionary \) and concave w.r.t. \( \separator \). Thus, the saddle point solution can be computed efficiently by simple ascent-descent type iterations. The novelty and the convergence attributes of learning a dictionary to solve \eqref{eq:dictionary-learning-samples} can be attributed to the reformulation \eqref{eq:coding-problem-primal-dual} of the LIP provided in this article.

\section{Theory, discussion and proofs.}
\label{section:convex-geometry}
In this section we investigate/study the linear inverse problem \eqref{eq:coding-problem} in detail and its dual with special emphasis on the underlying convex geometry. Based on the principle of separation of convex bodies by linear functionals, we obtain the dual problem \eqref{eq:coding-problem-equivalent-sup-problem} of the linear inverse problem which then leads to the convex-concave min-max problem \eqref{eq:coding-problem-equivalent-sup-problem}. We later provide the proof of Theorem \ref{} establishing that the optimal value of this min-max problem is proportional to the optimal cost \( \samplecost \) of LIP \eqref{eq:coding-problem}.

\begin{lemma}
For a linear map \( \linmap : \R{\dsize} \longrightarrow \hilbert \) and \( \regulizer, \scaling \geq 0 \), let \( \atomset{\scaling} \define \scaling \cdot \atomball \), we have
\begin{equation}
\label{eq:atomic-set-encoding-cost-equivalence}
\atomset{\scaling} = \{ z \in \hilbert : \encodedcost{\linmap}{z}{0}{\regulizer} \leq \scaling^{\horder} \}.
\end{equation}
\end{lemma}
\begin{proof}
On the one hand, it follows from the definition \eqref{eq:atomic-ball-definition} of \( \atomball \) that for every \( z \in \atomset{\scaling} \), there exists \( \repvec_z \in \R{\dsize} \) such that \( \cost (\repvec_z) \leq \scaling^{\horder} \) and \( \inpnorm{z - \linmap (\repvec_z)} \leq \regulizer \scaling \). Thus, considering \( \error = 0 \) in \eqref{eq:coding-problem}, we see that the pair \( (\scaling, \repvec_z) \) is a feasible point and hence we have \( \encodedcost{\linmap}{z}{0}{\regulizer} \leq \scaling^{\horder} \). 

On the other hand, if \( z \in \hilbert \) is such that \( \encodedcost{\linmap}{z}{0}{\regulizer} \leq \scaling^{\horder} \), we know that there exists a pair \( (\mathsf{c}_z, \repvec_z) \in \R{}_+ \times \R{\dsize} \) such that \( \mathsf{c}_z^{\horder} = \encodedcost{\linmap}{z}{0}{\regulizer} \leq \scaling^{\horder} \) and satisfies the following:
\begin{itemize}
    \item \( \cost (\repvec_z) \leq \mathsf{c}_z^{\horder} \leq \scaling^{\horder} \), and
    \item \( \inpnorm{z - \linmap (\repvec_z)} \leq 0 + \regulizer \mathsf{c}_z \leq \regulizer \scaling \).
\end{itemize}
It then immediately follows that for every \( z \in \hilbert \) satisfying \( \encodedcost{\linmap}{z}{0}{\regulizer} \leq \scaling^{\horder} \), we have the membership \( z \in \atomset{\scaling} \). Collecting the two assertions we arrive at \eqref{eq:atomic-set-encoding-cost-equivalence}.
\end{proof}

It is easy to see that \( \atomball = \bigcup\limits_{h \in \costatomset} \closednbhood{\linmap (h)}{\regulizer} \). Thus, if \( \regulizer > 0 \), the set \( \atomball \) has non-empty interior, and is therefore an absorbing set of \( \hilbert \).\footnote{A set \( S \) is an absorbing set of a vector space \( H \) if forevery \( z \in H \) there exists \( \scaling_z \geq 0 \) such that \( z \in \scaling_z \cdot S \).} When \( \regulizer = 0 \), we immediately see that \( \atomset{\scaling} \subset \image (\linmap) \) for every \( \scaling \geq 0 \). Furthermore, for every \( z \in \image (\linmap) \) we know that there exists \( \repvec \in \R{\dsize} \) such that \( z = \linmap (\repvec) \), and therefore \( z \in S_0{(\linmap , (\cost (\repvec))^{1/\horder}}) \). Consequently, we obtain:
\[
\lim_{\scaling \to +\infty} \atomset{\scaling} = 
\begin{cases}
\begin{aligned}
& \hilbert               && \regulizer > 0 \; , \\
& \image (\linmap)   && \regulizer = 0 \; .    
\end{aligned}
\end{cases}
\]

As the scaling factor \( \scaling \) increases, the set \( \atomset{\scaling} \) scales linearly by absorbing every \( \Depsdelta \)-feasible point in the set \( \hilbert \). In particular, the set \( \closednbhood{\samplevec}{\error} \) eventually intersects with \( \atomset{\scaling} \) for some \( \scaling \geq 0 \). We shall see that the optimal cost \( \samplecost \) is proportional to the minimum amount by which the set \( \atomset{1} \) needs to be scaled so that it intersects with \( \closednbhood{\samplevec}{\error} \).

\begin{lemma}
\label{lemma:coding-problem-reformulation-1}
For a given linear map \( \linmap : \R{\dsize} \longrightarrow \hilbert \) and non-negative real numbers \( \error ,\regulizer \), let \( \samplevec \in \hilbert \) be  \( \Depsdelta \)-feasible in the sense of Definition \ref{def:representability}, then we have
\begin{equation}
\label{eq:lemma:coding-problem-reformulation-1-1}
\samplecost =
\begin{cases}
\begin{aligned}
& \min_{\scaling \geq 0} && \scaling^{\horder} \\
& \sbjto && \atomset{\scaling} \; \cap \; \closednbhood{\samplevec}{\error} \neq \emptyset .
\end{aligned}
\end{cases}
\end{equation}
\end{lemma}
\begin{proof}
Let \( \scaling \geq 0 \) be such that \( \atomset{\scaling} \cap \closednbhood{\samplevec}{\error} \neq \emptyset \). Then on the one hand, there exists \( y_{\scaling} \in\closednbhood{\samplevec}{\error} \) and \( \repvec_{\scaling} \in \R{\dsize} \) such that \( \inpnorm{y_{\scaling} - \linmap (\repvec_{\scaling})} \leq \regulizer \scaling \) and \( \cost (\repvec_{\scaling}) \leq \scaling^{\horder} \). From this we get
\[
\inpnorm{\samplevec - \linmap (\repvec_{\scaling})} \leq \inpnorm{\samplevec - y_{\scaling}} + \inpnorm{y_{\scaling} - \linmap (\repvec_{\scaling})} \leq \error + \regulizer \scaling \; ,
\]
which implies that the pair \( (\scaling , \repvec_{\scaling}) \) is feasible for \eqref{eq:coding-problem}, and as a result we get \( \samplecost \leq \scaling^{\horder} \). By minimizing over \( \scaling \geq 0 \) such that \( \atomset{\scaling} \cap \closednbhood{\samplevec}{\error} \neq \emptyset \) we get our first inequality:
\[
\samplecost \leq 
\begin{cases}
\begin{aligned}
& \min_{\scaling \geq 0} && \scaling^{\horder} \\
& \sbjto && \atomset{\scaling} \; \cap \; \closednbhood{\samplevec}{\error} \neq \emptyset .
\end{aligned}
\end{cases}
\]

On the other hand, for every pair \( (\scaling, \repvec) \) that is feasible for \eqref{eq:coding-problem}, by defining \( y \coloneqq \samplevec \indicator_{\{ 0 \}} (\error) + \frac{\error \linmap (\repvec) + \regulizer \scaling \samplevec}{\error + \regulizer \scaling} \; \indicator_{ ]0, +\infty[ } (\error) \) we shall establish that \( y \in \closednbhood{\samplevec}{\error} \cap \atomset{\scaling} \).

Whenever \( \error = 0 \) we have \( y = \samplevec \), and from the feasibility of the pair \( (\scaling, \repvec) \) it easily follows that \( \inpnorm{y - \linmap (\repvec)} = \inpnorm{\samplevec - \linmap (\repvec)} \leq \error + \regulizer \scaling = \regulizer \scaling \) and \( \cost (\repvec) \leq \scaling^{\horder} \). Thus, \( \encodedcost{y}{\linmap}{0}{\regulizer} \leq \scaling^{\horder} \) and from \eqref{eq:atomic-set-encoding-cost-equivalence} the membership \( y \in \atomset{\scaling} \) holds. Similarly, if \( \error = 0 \), we see that \( y = \samplevec = \closednbhood{\samplevec}{0} \). Therefore, \( y \in \closednbhood{\samplevec}{\error} \cap \atomset{r} \), and the intersection is non-empty.

When \( \error > 0 \), we see that
\[
\begin{aligned}
    & \inpnorm{ \samplevec - y } = \inpnorm{ \samplevec - \frac{\error \linmap (\repvec) + \regulizer \scaling \samplevec}{\error + \regulizer \scaling} } = \frac{\error}{\error + \regulizer \scaling} \inpnorm{ \samplevec - \linmap (\repvec)} \leq \error \; , \text{ and } \\
    & \inpnorm{y - \linmap (\repvec)} = \inpnorm{ \frac{\error \linmap (\repvec) + \regulizer \scaling \samplevec}{\error + \regulizer \scaling} - \linmap (\repvec) } = \frac{ \regulizer \scaling }{\error + \regulizer \scaling} \inpnorm{ \samplevec - \linmap (\repvec)} \leq \regulizer \scaling \; .
\end{aligned}
\]
These inequalities, along with the fact that \( \cost (\repvec) \leq \scaling^{\horder} \) imply that \( y \in \closednbhood{\samplevec}{\error} \cap \atomset{\scaling} \) and in particular that \( \closednbhood{\samplevec}{\error} \cap \atomset{\scaling} \neq \emptyset \). As a consequence, the inequality:
\[
\scaling \geq 
\begin{cases}
\begin{aligned}
& \min_{\scaling \geq 0} && \scaling^{\horder} \\
& \sbjto && \atomset{\scaling} \; \cap \; \closednbhood{\samplevec}{\error} \neq \emptyset ,
\end{aligned}
\end{cases}
\]
holds for every pair \( (\scaling, \repvec) \) that is feasible for \eqref{eq:coding-problem}. By minimizing over all the pairs \( (\scaling, \repvec) \) that are feasible for \eqref{eq:coding-problem}, we obtain the converse inequality
\[
\samplecost \geq 
\begin{cases}
\begin{aligned}
& \min_{\scaling \geq 0} && \scaling^{\horder} \\
& \sbjto && \atomset{\scaling} \; \cap \; \closednbhood{\samplevec}{\error} \neq \emptyset .
\end{aligned}
\end{cases}
\]
This completes the proof.
\end{proof}

\begin{remark}
An interesting viewpoint to  take from this in dictionary learning problem is that every dictionary \( \dictionary \) gives rise to an atomic set \( S_{\regulizer} (\dictionary, 1) \), and the encoding cost \( \encodedcost{\dictionary}{\samplevec}{\error}{\regulizer} \) of a vector \( \samplevec \) is proportional to the \emph{approximate} Minkowski gauge function with respect to this set.\footnote{We say ``approximate'' in the sense that we do not scale the atomic set \( S_{\regulizer}(\dictionary , 1) \) so as to absorb \( \samplevec \). Instead, we scale it only until it intersects with a given neighborhood of \( \samplevec \). } The corresponding dictionary learning problem can be viewed as the task of finding a `good' atomic set arising from a dictionary.
\end{remark}

\subsubsection{Intersection of the convex bodies.}
Lemma \ref{lemma:coding-problem-reformulation-1} gives us the first required connection between the LIP \eqref{eq:coding-problem} and the underlying convex geometry. It asserts that the value \( (\samplecost)^{1/\horder} \) is the minimum amount by which the set \( \atomball \) has to be scaled linearly so that it intersects with \( \closednbhood{\samplevec}{\error} \). To this end, let us define:
\[
\atomscaled \coloneqq (\samplecost)^{1/\horder} \cdot \atomball \; .
\]
We observe that both the sets \( \atomscaled \) and \( \closednbhood{\samplevec}{\error} \) are compact and convex, and due to this, we have the following intersection lemma.
\begin{lemma}
\label{lemma:uniqueness-of-intersection}
Let \( \samplevec \in \hilbert \) be \( \Depsdelta \)-feasible in the sense of Definition \ref{def:representability}. Let \( (\mathsf{c}_{\samplevec}, \repvec_{\samplevec}) \in [0, +\infty[ \times \R{\dsize} \) be an optimal solution to the coding problem \eqref{eq:coding-problem}, i.e., \( \mathsf{c}_{\samplevec} = ( \samplecost )^{1/\horder} \) and \( \repvec_{\samplevec} \in \codes \). Then the sets \( \closednbhood{\samplevec}{\error} \) and \( \atomscaled \) intersect at a unique point \( y\opt \) given by:
\begin{equation}
\label{eq:intersection-lemma}
    \closednbhood{\samplevec}{\error} \cap \atomscaled \; \eqqcolon \; y\opt \; = \;  \samplevec \indicator_{ \{ 0 \} } (\error) \; + \; \frac{\error \linmap (\repvec_{\samplevec}) + \regulizer \mathsf{c}_{\samplevec} \samplevec}{\error + \regulizer \mathsf{c}_{\samplevec}} \; \indicator_{ ]0, +\infty[ } (\error) .
\end{equation}
As a consequence, we assert that
\begin{itemize}
    \item whenever, \( \inpnorm{\samplevec} > \error \), every \( \repvec_{\samplevec} \in \codes \) satisfies \( \inpnorm{\samplevec - \linmap (\repvec_{\samplevec})} = \error + \regulizer \mathsf{c}_{\samplevec} \) ;

    \item for every \( \repvec_{\samplevec} , g_{\samplevec} \in \codes \), we have \( \linmap (\repvec_{\samplevec}) \; = \; \linmap (g_{\samplevec}) \) .
\end{itemize}
\end{lemma}
\begin{remark}
In dictionary learning, for a given dictionary \( \dictionary \), since every sample vector \( \samplevec \in \hilbert \) is represented by some vector \( \repvec_{\dictionary} (\samplevec) \in \encodermap{\dictionary}{\samplevec}{\error}{\regulizer} \), the representation is not unique whenever the set \( \encodermap{\dictionary}{\samplevec}{\error}{\regulizer} \) is not a singleton. In such situations, even though the representation need not be unique, we emphasize that the reconstruction \( \samplevec_{\text{rec}} \coloneqq \dictionary \repvec_{\dictionary}( \samplevec ) \) of the vector \( \samplevec \) obtained from its representation \( \repvec_{\dictionary} (\samplevec) \) is unique.
\end{remark}
\begin{remark}
\label{remark:alternate-expression-for-intersection-point}
When \( \samplevec \notin \closednbhood{0}{\error} \) since \( \mathsf{c}_{\samplevec} > 0 \), it is easily verified that the unique point of intersection \( y \) in Lemma \ref{lemma:uniqueness-of-intersection} can also be written as:
\[
y\opt = \linmap (\repvec_{\samplevec}) + \frac{\regulizer \mathsf{c}_{\samplevec}}{\error + \regulizer \mathsf{c}_{\samplevec}} \big( \samplevec - \linmap (\repvec_{\samplevec}) \big) \indicator_{ ]0, +\infty[ } (\regulizer) .
\]
\end{remark}
\begin{proof}[Proof of Lemma \ref{lemma:uniqueness-of-intersection}]
We note that if \( \error = 0 \), \( \closednbhood{\samplevec}{\error} = \samplevec \), and since by definition, the set \( \atomscaled \) intersects with \( \closednbhood{\samplevec}{\error} \). The intersecton happens at the point \( \samplevec \) which is unique. We shall establish \eqref{eq:intersection-lemma} by considering the remaining cases.
\begin{itemize}[leftmargin = *]
    \item \( 0 < \inpnorm{\samplevec} \leq \error \) : From \eqref{eq:zero-approximate-sample}, we know that \( \samplecost = 0 \) and \( \codes = \{ 0 \} \). This implies that  \( \atomscaled = \{ 0 \} \). In addition, we see that \( 0 \in \closednbhood{\samplevec}{\error} \) whenever \( \inpnorm{\samplevec} \leq \error \). As a result, we obtain that \( \closednbhood{\samplevec}{\error} \cap \atomscaled = \{ 0 \} \). Now, by using the fact that \( (\mathsf{c}_{\samplevec} , \repvec_{\samplevec}) \) is an optimal solution to the coding problem \eqref{eq:coding-problem} if and only if \( (\mathsf{c}_{\samplevec} , \repvec_{\samplevec}) = (0,0) \), we see that \( y\opt \) in \eqref{eq:intersection-lemma} evaluates to \( 0 \) confirming \eqref{eq:intersection-lemma}.
    
    \item \( 0 < \error < \inpnorm{\samplevec} \) : We shall prove by contradiction that the sets \( \closednbhood{\samplevec}{\error} \) and \( \atomscaled \) intersect at a unique point. Let \( y_1 \neq y_2 \) be such that \( y_1 , y_2 \in \closednbhood{\samplevec}{\error} \cap \atomscaled \). Since \( \closednbhood{\samplevec}{\error} \) is a strictly convex set, \( \frac{1}{2} (y_1 + y_2) \in \nbhood{\samplevec}{\error} \). However, since \( \nbhood{\samplevec}{\error} \) is an open set, one can find \( \rho > 0 \) such that \( \closednbhood{ \frac{1}{2} (y_1 + y_2) }{\rho} \subset \nbhood{\samplevec}{\error} \). Since \( 0 \notin \nbhood{\samplevec}{\error} \), we conclude that \( 2\rho < \inpnorm{y_1 + y_2} \); Defining \( \theta \define \left( 1 - \frac{2 \rho}{\inpnorm{y_1 + y_2}} \right) \), we see that \( \theta \in ]0, 1[ \). It is easily verified that \( \inpnorm{\frac{1}{2}(y_1 + y_2) - \frac{\theta}{2} (y_1 + y_2) } = \rho \), which leads us to the first inclusion \( \frac{\theta}{2} (y_1 + y_2) \in \closednbhood{\frac{1}{2} (y_1 + y_2)}{\rho} \subset \closednbhood{\samplevec}{\error} \). In addition, we note that the set \( \atomscaled \) is also convex, which means that \( \frac{1}{2} (y_1 + y_2) \in \atomscaled \). Since \( \atomscaled \) scales linearly, we conclude that \( \frac{\theta}{2} (y_1 + y_2) \in \theta \cdot \atomscaled \). From these two inclusions, it is clear that
    \[
    \frac{\theta}{2} (y_1 + y_2) \in \closednbhood{\samplevec}{\error} \cap \theta \cdot \atomscaled = \closednbhood{\samplevec}{\error} \cap \atomset{\theta \mathsf{c}_{\samplevec} } \; ,
    \]
    and equivalently, \( \closednbhood{\samplevec}{\error} \cap \atomset{\theta \mathsf{c}_{\samplevec} } \neq \emptyset \). This, however, contradicts the assertion of Lemma \ref{lemma:coding-problem-reformulation-1} since \( \theta < 1 \). 
    
    To summarize, we have established that if the intersection of the sets \( \closednbhood{\samplevec}{\error} \) and \( \atomscaled \) is not a singleton, we can slightly shrink the set \( \atomscaled \) such that it still intersects \( \closednbhood{\samplevec}{\error} \) nontrivially. This is a contradiction in view of Lemma \ref{lemma:coding-problem-reformulation-1}.
    
    To prove that \( y\opt \) defined in \eqref{eq:intersection-lemma} is indeed the unique point of intersection, it suffices to show that \( y\opt \in \closednbhood{\samplevec}{\error} \cap \atomscaled \). We observe that:
    \[
\begin{aligned}
    \inpnorm{ \samplevec - y\opt } & = \inpnorm{ \samplevec - \frac{\error \linmap (\repvec_{\samplevec}) + \regulizer \mathsf{c}_{\samplevec} \samplevec}{\error + \regulizer \mathsf{c}_{\samplevec}} } = \frac{\error}{\error + \regulizer \mathsf{c}_{\samplevec}} \inpnorm{ \samplevec - \linmap (\repvec_{\samplevec})} \leq \error \; , \text{ and } \\
    \inpnorm{y\opt - \linmap (\repvec_{\samplevec}) } & = \inpnorm{ \frac{\error \linmap (\repvec_{\samplevec}) + \regulizer \mathsf{c}_{\samplevec} \samplevec}{\error + \regulizer \mathsf{c}_{\samplevec}} - \linmap (\repvec_{\samplevec}) } = \frac{ \regulizer \mathsf{c}_{\samplevec} }{\error + \regulizer \mathsf{c}_{\samplevec}} \inpnorm{ \samplevec - \linmap (\repvec_{\samplevec}) } \leq \regulizer \mathsf{c}_{\samplevec} .
\end{aligned}
    \]
These inequalities, along with the fact that \( \cost (\repvec_{\samplevec}) \leq \mathsf{c}_{\samplevec}^{\horder} \), imply that \( y\opt \in \closednbhood{\samplevec}{\error} \cap \atomscaled \). This establishes \eqref{eq:intersection-lemma}.
\end{itemize}

We proceed to establish the two consequences. to see the first, let us prove that the error constraint is active at the optimal solution \( (\mathsf{c}_{\samplevec}, \repvec_{\samplevec}) \) whenever \( \inpnorm{\samplevec} > \error \geq 0 \). If \( \error = \regulizer = 0 \), then the error constraint is trivially active since \( \inpnorm{\samplevec - \linmap (\repvec_{\samplevec}) } \leq 0 \) implies that \( \inpnorm{\samplevec - \linmap (\repvec_{\samplevec}) } = 0 \). If at least one of the parameters \( \error \) and \( \regulizer \) is positive, we know that \( \mathsf{c}_{\samplevec} > 0 \) for every \( \inpnorm{\samplevec} > \error \). Therefore, we have \( \error + \regulizer \mathsf{c}_{\samplevec} > 0 \), and the quantity \( y\opt \define \frac{\error \linmap (\repvec_{\samplevec}) + \regulizer \mathsf{c}_{\samplevec} \samplevec}{\error + \regulizer \mathsf{c}_{\samplevec}} \) is well defined for every \( \repvec_{\samplevec} \in \codes \). We know from the previous assertion of the lemma that \( y\opt \in \atomscaled \). However, since \( \atomscaled \) is a convex set that contains 0, we conclude that
\[
\theta y\opt \in \atomscaled  \text{ for every \( \theta \in [0,1] \)} .
\]
If we suppose that \( \inpnorm{\samplevec - \linmap (\repvec_{\samplevec}) } < \error + \regulizer \mathsf{c}_{\samplevec} \), it is easily verified that \( \inpnorm{\samplevec - y\opt} < \error \), and thus \( y\opt \in \nbhood{\samplevec}{\error} \). As a result, one can find \( \rho > 0 \) such that \( \closednbhood{y\opt}{\rho} \subset \closednbhood{\samplevec}{\error} \). Since \( 0 \notin \closednbhood{\samplevec}{\error} \), we see at once that \( \rho < \inpnorm{y\opt} \), and conclude that
\[
\alpha y\opt \in \closednbhood{y\opt}{\rho} \subset \closednbhood{\samplevec}{\error} \text{ for every \( \alpha \) such that \( \left( 1 - \frac{\rho}{\inpnorm{y\opt}} \right) \leq \alpha \leq 1 \)} .
\]
These two inclusions together contradict that the sets \( \closednbhood{\samplevec}{\error} \) and \( \atomscaled \) intersect at a unique point.

It remains to prove the final assertion that for every \( \repvec_{\samplevec} , g_{\samplevec} \in \codes \), the equality \( \linmap (\repvec_{\samplevec}) = \linmap ( g_{\samplevec} ) \) holds. Indeed, whenever \( \inpnorm{\samplevec} \leq \error \), we have \( \samplecost = 0 \) and \( \codes = \{ 0 \} \). This implies that \( \repvec_{\samplevec} = g_{\samplevec} = 0 \), and thus \( \linmap (\repvec_{\samplevec}) = \linmap ( g_{\samplevec} ) \). Let us consider the case when \( \inpnorm{\samplevec} > \error \), and suppose that \( \linmap (\repvec_{\samplevec}) \neq \linmap (g_{\samplevec}) \) for some \( \repvec_{\samplevec} , g_{\samplevec} \in \codes \). Then it follows that \( \frac{1}{2} (\repvec_{\samplevec} + g_{\samplevec}) \) satisfies the error constraint
\[
\inpnorm{\samplevec - \frac{1}{2} \linmap (\repvec_{\samplevec} + g_{\samplevec})} = \inpnorm{\frac{1}{2} \big(\samplevec - \linmap (\repvec_{\samplevec}) \big) +  \frac{1}{2} \big( \samplevec - \linmap (g_{\samplevec}) \big)} \leq \error + \regulizer \mathsf{c}_{\samplevec} .
\]
Moreover, we know that the level sets of  \( \cost \) are convex and since \( \repvec_{\samplevec} , g_{\samplevec} \in \mathsf{c}_{\samplevec} \cdot \costatomset \), we have \( \frac{1}{2} (\repvec_{\samplevec} + g_{\samplevec}) \in \mathsf{c}_{\samplevec} \). Therefore, \( \cost (\frac{1}{2} (\repvec_{\samplevec} + g_{\samplevec}) ) \leq \mathsf{c}_{\samplevec}^{\horder} = \samplecost \), we conclude that \( \frac{1}{2} (\repvec_{\samplevec} + g_{\samplevec}) \in \codes \). However, since \( \linmap (\repvec_{\samplevec}) \neq \linmap (g_{\samplevec}) \), the triangle inequality implies that the above error constraint is satisfied strictly. This contradicts our earlier assertion that the error constraint is active for every \( \repvec_{\samplevec} \in \codes \). The proof is complete.
\end{proof}

\begin{lemma}
\label{lemma:separation-dummy}
Let the linear map \( \linmap : \R{\dsize} \longrightarrow \hilbert \) and non-negative real numbers \( \error, \regulizer \) be given, then for every \( \separator \in \hilbert \), we have
\begin{equation}
\label{eq:maximization-over-atomball-changed-to-hvar}
\dualnorm{\separator} \define \max_{z \in \atomball} \inprod{\separator}{z} \ = \ \regulizer \inpnorm{\separator} \; + \; \max_{\hvar \in \costatomset} \ \inprod{\separator}{\linmap (\hvar)} .
\end{equation}
Furthermore,
\begin{enumerate}[leftmargin = * , label = \rm{(\roman*)}]
\item If \( \regulizer > 0 \), then \( \dualnorm{\separator} > 0 \) for every \( \separator \in \hilbert \setminus \{ 0 \} \).

\item If \( \regulizer = 0 \), and \( \separator \in \hilbert \setminus \{ 0 \} \) satisfies \( \dualnorm{\separator} = 0 \), then \( \inprod{\separator}{\samplevec} - \error \inpnorm{\separator} \leq 0 \) for every \( (\linmap , \error , 0) \)-feasible vector \( \samplevec \in \hilbert \).
\end{enumerate}
\end{lemma}
\begin{proof}
We recall from the definition \eqref{eq:atomic-ball-definition} that the set \( \atomball \) is the image of the linear map: \( \closednbhood{0}{\regulizer} \times \costatomset \ni (z', \hvar) \longmapsto z' + \linmap (\hvar) \). This allows us to write the optimization problem \( \max\limits_{z \in \atomball} \inprod{\separator}{z} \) equivalently as:
\[
\max\limits_{\hvar, \; z'} \ \ \inprod{\separator}{z' + \linmap (\hvar)} \quad \sbjto \ \ \hvar \in \costatomset , \ z' \in \closednbhood{0}{\regulizer} .
\]
It is easily seen that the above optimization problem is separable into maximization over individual variables, and using the fact that \( \max\limits_{z' \in \closednbhood{0}{\regulizer}} \inprod{\separator}{z'} = \regulizer \inpnorm{\separator} \) for every \( \separator \in \hilbert \) \eqref{eq:maximization-over-atomball-changed-to-hvar} follows at once. Moreover, since \( 0 \in \costatomset \), we have \( 0 \leq \max\limits_{\hvar \in \costatomset} \; \inprod{\separator}{\linmap (\hvar)} \) for every \( \separator \in \hilbert \). Applying this inequality in \eqref{eq:maximization-over-atomball-changed-to-hvar}, assertion (i) of the lemma follows immediately.

Finally, let \( \regulizer = 0 \) and \( \separator \in \hilbert \setminus \{ 0 \} \) satisfy \( \dualnorm{\separator} = 0 \). Since \( S_0 (\linmap , 1) \) is an absorbing set to \( \image (\linmap) \), we conclude from the definition \eqref{eq:dual-norm-definition} of the dual function that \( \inprod{\separator}{y} \leq 0 \) for every \( y \in \image (\linmap) \). If \( \samplevec \in \hilbert \) is \( (\linmap, \error, 0) \)-feasible, we know that \( \closednbhood{\samplevec}{\error} \cap \image (\linmap) \neq \emptyset \). Let \( y' \in \closednbhood{\samplevec}{\error} \cap \image (\linmap) \), then
\[
\inprod{\separator}{\samplevec} - \error \inpnorm{\separator} = \min_{y \in \closednbhood{\samplevec}{\error}} \inprod{\separator}{y} \ \leq \ \inprod{\separator}{y'} \; \leq \; 0 .
\]
This completes the proof.
\end{proof}

\subsubsection{Separation of sets \( \closednbhood{\samplevec}{\error} \) and \( \atomscaled \).}
We recall that both the sets \( \closednbhood{\samplevec}{\error} \) and \( \atomscaled \) are compact convex subsets that intersect at the unique point \( y\opt \). As a result, we know from the Hahn-Banach separation principle that there exists a \( \separator\opt \in \hilbert \) such that the linear functional \( \inprod{\separator\opt}{\cdot} \) satisfies
\begin{equation}
\label{eq:the-other-inequality }
\max_{z \in \atomscaled} \inprod{\separator\opt}{z} \ = \ \inprod{\separator\opt}{y\opt} \ = \ \min_{y \in \closednbhood{\samplevec}{\error}} \inprod{\separator\opt}{y} .
\end{equation}
In other words, the linear functional \( \inprod{\separator\opt}{\cdot} \), separates the convex sets \( \closednbhood{\samplevec}{\error} \) and \( \atomscaled \), and supports them at their unique point of intersection \( y\opt \). This fact, is central in establishing strong duality and explicitly characterizing the optimal dual variables.

\begin{lemma}
\label{lemma:equality-condition}
Consider \eqref{eq:the-other-inequality } where at least one of \( \error, \regulizer \) is positive. If  \( 0 \neq \separator' \in \hilbert \) satisfies \eqref{eq:the-other-inequality }, then \( \separator' = \alpha \big(\samplevec - \linmap (\repvec_{\samplevec}) \big) \) for some  \( \alpha > 0 \) and \( \repvec_{\samplevec} \in \codes \). Consequently, \eqref{eq:the-other-inequality } is satisfied by \( \alpha (\samplevec - \linmap (\repvec_{\samplevec})) \) for every \( \alpha > 0 \).
\end{lemma}
\begin{proof}
We recall from the Remark \ref{remark:alternate-expression-for-intersection-point} that the sets \( \closednbhood{\samplevec}{\error} \) \( \atomscaled \) intersect at the unique point \( y\opt \), given by
\[
y\opt = \linmap (\repvec_{\samplevec})  + \frac{\regulizer \mathsf{c}_{\samplevec}}{\error + \regulizer \mathsf{c}_{\samplevec}} \big( \samplevec - \linmap (\repvec_{\samplevec}) \big) \indicator_{ ]0, +\infty[ } (\regulizer) ,
\]
where \( \mathsf{c}_{\samplevec} \define (\samplecost)^{1/\horder} \) and \( \repvec_{\samplevec} \in \codes \).
\begin{itemize}[leftmargin = *]
    \item On the one hand, if \( \error > 0 \) and \( \separator' \neq 0 \) satisfies: \( \inprod{\separator'}{y\opt} \ = \ \min\limits_{y \in \closednbhood{\samplevec}{\error}} \ \inprod{\separator'}{y} \), then necessarily \( \separator' = \alpha' (\samplevec - y\opt) \) for some \( \alpha' > 0 \).
    
    \item On the other hand, if \( \regulizer > 0 \), and \( \separator' \neq 0 \) satisfies: \( \inprod{\separator'}{y\opt} \ = \ \max\limits_{z \in \atomscaled} \ \inprod{\separator'}{z} \), then due to the fact that \( y\opt \in \closednbhood{\linmap (\repvec_{\samplevec}) }{\regulizer \mathsf{c}_{\samplevec}} \subset \atomscaled \) \( \separator' \) also satisfies: \( \inprod{\separator'}{y\opt} \ = \ \max\limits_{z \in \closednbhood{\linmap (\repvec_{\samplevec})}{\regulizer \mathsf{c}_{\samplevec}}} \ \inprod{\separator'}{z} \). It follows that: \( \separator' = \alpha'' \big( y\opt - \linmap (\repvec_{\samplevec}) \big) \) for some \( \alpha'' > 0 \). 
\end{itemize}
By substituting for \( y\opt \) and simplifying, we easily deduce that in both the cases \( \separator' = \alpha \big( \samplevec - \linmap (\repvec_{\samplevec}) \big) \) for some \( \alpha > 0 \). 

Suppose that \eqref{eq:the-other-inequality } is true for some \( \alpha' > 0 \), then for any \( \alpha > 0 \), the inequalities in \eqref{eq:the-other-inequality } are preserved by multiplying throughout by the positive quantity \( \frac{\alpha}{\alpha'} \). Thus, \eqref{eq:the-other-inequality } is satisfied by \( \alpha (\samplevec - \linmap (\repvec_{\samplevec})) \) for every \( \alpha > 0 \).
\end{proof}

\begin{lemma}
\label{lemma:optimal-separation-of-convex-bodies}
Let the linear map \( \linmap : \R{\dsize} \longrightarrow \hilbert \) and non-negative real numbers \( \error, \regulizer
\) be given, and \( \samplevec \in \hilbert \setminus \closednbhood{0}{\error} \) be any \( \Depsdelta \)-feasible vector such that \( \separatorset \neq \emptyset \). Then every \( \separator\opt \in \separatorset \) satisfies \eqref{eq:the-other-inequality }.
\end{lemma}
\begin{proof}
We first recall that \( \atomscaled = (\samplecost)^{1/\horder} \cdot \atomball \). Thus, for every \( \separator\opt \in \separatorset \), the following relations hold: 
\begin{align*}
& \max\limits_{z \in \atomscaled} \ \inprod{\separator\opt}{z} \ = \ (\samplecost)^{1/\horder} \max\limits_{z \in \atomball} \ \inprod{\separator\opt}{z} \ = \ (\samplecost)^{1/\horder} , \text{ and} \\
& \min\limits_{y \in \closednbhood{\samplevec}{\error}} \ \inprod{\separator\opt}{y} \ = \ \inprod{\separator\opt}{\samplevec} - \error \inpnorm{\separator\opt} \ = \ (\samplecost)^{1/\horder} .
\end{align*}
In other words, the linear functional \( \inprod{\separator\opt}{\cdot} \) separates the sets \( \closednbhood{\samplevec}{\error} \) and \( \atomscaled \). Moreover, both these sets are compact and convex, and we know from Lemma \ref{lemma:uniqueness-of-intersection} that they intersect at a unique point \( y\opt \). Therefore, the linear functional \( \inprod{\separator\opt}{\cdot} \) must support both these sets at their intersection point \( y\opt \), and \eqref{eq:the-other-inequality } follows at once.
\end{proof}

\begin{proposition}
\label{proposition:description-of-the-separator-set}
Let the linear map \( \linmap : \R{\dsize} \longrightarrow \hilbert \) and non-negative real numbers \( \error, \regulizer \) be given, and \( \samplevec \in \hilbert \setminus \closednbhood{0}{\error} \) be a \( \Depsdelta \)-feasible in the sense of Def. \ref{def:representability}. The set \( \separatorset \) is completely described in the following.
\begin{enumerate}[label = \rm{(\roman*)}, leftmargin=*]
\item If \( \regulizer = 0 \), \( \error = 0 \), then the set \( \Lambda_0 (\linmap, \samplevec, 0) \neq \emptyset \), and in particular, \( \Lambda_0 (\linmap, \samplevec, 0) \cap \image (\linmap) \neq \emptyset \). A vector \( \separator\opt \in \Lambda_0 (\linmap, \samplevec, 0) \) if and only if the linear functional \( \inprod{\separator\opt}{\cdot} \) supports the set \( S_0 (\linmap, \samplevec, 0) \) at \( \samplevec \), and satisfies \( \dualnorm{\separator\opt} = 1 \).\footnote{If \( \image (\linmap) \) is a proper subspace of \( \hilbert \), then every \( \separator \) in the orthogonal complement of \( \image (\linmap) \) supports the set \( S_0 (\linmap , 1) \) at every point, and in particular at \( \samplevec \). However, such a \( \separator \) doesn't satisfy the condition \( \inprod{\separator}{\samplevec} - \error \inpnorm{\separator} = (\samplecost)^{1/\horder} \).}

\item If at least one of the following is true 
\begin{itemize}
    \item \( \regulizer > 0 \)
    
    \item \( \regulizer = 0 \) and \( \error > 0 \) with \( \nbhood{\samplevec}{\error} \cap \image (\linmap) \neq \emptyset \)
\end{itemize}
then the set \( \separatorset \) consists of a unique element \( \separator\opt \) given by
    \begin{equation}
    \label{eq:unique-separator-solution}
          \separator\opt = \frac{\samplevec - \linmap (\repvec_{\samplevec}) }{ \dualnorm{\samplevec - \linmap (\repvec_{\samplevec}) } } \quad \text{for any \( \repvec_{\samplevec} \in \codes \) }.\footnote{Even though the set \( \codes \) may contain multiple elements, \( \separator\opt \) is unique due to the fact that \( \linmap (\repvec_{\samplevec}) \) is unique.}
    \end{equation}
    
\item If \( \regulizer = 0 \) and \( \error > 0 \) such that \( \nbhood{\samplevec}{\error} \cap \image (\linmap) = \emptyset \), then \( \separatorset = \emptyset \).
\end{enumerate}
\end{proposition}
\begin{proof}
If \( \error = \regulizer = 0 \), then the set \( \closednbhood{\samplevec}{\error} = \{ \samplevec \} \). Thus, in view of Lemma \ref{lemma:optimal-separation-of-convex-bodies} we know that \( \separator\opt \in \Lambda_0 (\linmap, \samplevec, 0) \) if and only if the linear functional \( \inprod{\separator\opt}{\cdot} \) supports the set \( S_0 (\linmap, \samplevec, 0) \) at \( \samplevec \), and satisfies \( \dualnorm{\separator\opt} = 1 \). It remains to be shown that the set \( \Lambda_0 (\linmap, \samplevec, 0) \) is non-empty, and we do so by showing that there exists \( \separator_{\linmap} \in \Lambda_0 (\linmap, \samplevec, 0) \cap \image (\linmap) \). Since \( \samplevec \) is \( (\linmap, 0, 0) \)-feasible, we have \( \samplevec \in \image (\linmap) \). We note from Lemma \ref{lemma:coding-problem-reformulation-1} that \( \encodedcost{\linmap}{\samplevec}{0}{0} \) is the least amount by which the set \( S_0 (\linmap,1) \) has to be linearly scaled so that it contains \( \samplevec \). This implies that \( \samplevec \) lies on the boundary of the set \( S_0 (\linmap, \samplevec, 0) \), i.e., \( \samplevec \notin \relinterior (S_0 (\linmap, \samplevec, 0)) \). In addition, since \( S_0 (\linmap, \samplevec, 0) \) is a convex subset of \( \image(\linmap) \), we know that there exists \( 0 \neq \separator_{\linmap} \in \image(\linmap) \) such that the linear functional \( \inprod{\separator_{\linmap}}{\cdot} \) supports the set \( S_0 (\linmap, \samplevec, 0) \) at the boundary point \( \samplevec \). As  result, we obtain:
\[
\begin{aligned}
\inprod{\separator_{\linmap}}{\samplevec} & = \max\limits_{z \in S_0 (\linmap, \samplevec, 0)} \inprod{\separator_{\linmap}}{z} = (\samplecost)^{1/\horder} \max\limits_{z \in S_0 (\linmap, 1)} \inprod{\separator_{\linmap}}{z} \\
& = (\samplecost)^{1/\horder} \dualnorm{\separator_{\linmap}} .
\end{aligned}
\]
Since \( S_0 (\linmap, 1) \) is an absorbing set to \( \image (\linmap) \) we have \( 0 \in \relinterior S_0 (\linmap, 1) \) and therefore \( 0 <  \dualnorm{\separator_{\linmap}} \). Thus, defining \( \separator\opt \define (1/ \dualnorm{\separator_{\linmap}})\separator_{\linmap} \) it readily follows that \( \separator\opt \in  \Lambda_0 (\linmap, \samplevec, 0) \). This establishes the assertion (i) of the proposition.

If either \( \error > 0 \) or \( \regulizer > 0 \), on the one hand we know from Lemma \ref{lemma:equality-condition} that \( ( \samplevec - \linmap (\repvec_{\samplevec}) ) \) satisfies
\[
\inprod{\samplevec - \linmap (\repvec_{\samplevec})}{\samplevec} - \error \inpnorm{\samplevec - \linmap (\repvec_{\samplevec})} = \max_{z \in \atomscaled} \inprod{\samplevec - \linmap (\repvec_{\samplevec})}{z} = (\samplecost)^{1/\horder} \dualnorm{\samplevec - \linmap (\repvec_{\samplevec})} .
\]
We immediately see that if \( \dualnorm{\samplevec - \linmap (\repvec_{\samplevec})} > 0 \), then \( \frac{\samplevec - \linmap (\repvec_{\samplevec})}{\dualnorm{\samplevec - \linmap (\repvec_{\samplevec})}} \in \separatorset \). On the other hand, if \( \separator\opt \in \separatorset \) then Lemma \ref{lemma:optimal-separation-of-convex-bodies} implies that \( \separator\opt \) must satisfy \eqref{eq:the-other-inequality }, and from Lemma \ref{lemma:equality-condition} we infer that \( \separator\opt = \alpha (\samplevec - \linmap (\repvec_{\samplevec})) \) for some \( \alpha > 0 \). From Definition \ref{def:optimal-separator-set} it immediately implies that if \(  \alpha (\samplevec - \linmap (\repvec_{\samplevec})) \in \separatorset \), then \( \dualnorm{ \samplevec - \linmap (\repvec_{\samplevec})} > 0 \) and \( \alpha = \frac{1}{\dualnorm{ \samplevec - \linmap (\repvec_{\samplevec})}} \). Thus the set \( \separatorset \) is non-empty, and is the singleton \( \left\{ \frac{\samplevec - \linmap (\repvec_{\samplevec})}{\dualnorm{\samplevec - \linmap (\repvec_{\samplevec})}} \right\} \) if and only if \( \dualnorm{\samplevec - \linmap (\repvec_{\samplevec})} > 0 \). 

We complete the proof by showing that \( \dualnorm{ \samplevec - \linmap (\repvec_{\samplevec}) } = 0 \) if and only if \( \regulizer = 0 \) and \( \nbhood{\samplevec}{\error} \cap \image (\linmap) = \emptyset \). On the one hand, if \( \regulizer = 0 \) and \( \nbhood{\samplevec}{\error} \cap \image (\linmap) = \emptyset \), then we have 
\[
\inpnorm{\samplevec - \pi_{\linmap} (\samplevec)} = \min_{z \in \image (\linmap)} \inpnorm{\samplevec - z} \geq \error .
\]
However, from Lemma \ref{lemma:uniqueness-of-intersection} we know that \( \inpnorm{\samplevec - \linmap (\repvec_{\samplevec}) } = \error \), and since \( \linmap (\repvec_{\samplevec}) \in \image (\linmap) \), we deduce that \( \pi_{\linmap} (\samplevec) = \linmap (\repvec_{\samplevec}) \).\footnote{\( \pi_{\linmap} : \hilbert \longrightarrow \image (\linmap) \) is the orthogonal projection operator onto \( \image (\linmap) \).} Due to orthogonality of projection, \( \inprod{\samplevec - \linmap (\repvec_{\samplevec}) }{z} = 0 \) for all \( z \in \image (\linmap) \). Since \( S_0 (\linmap , 1 ) \subset \image(\linmap) \), we obtain
\[
\dualnorm{\samplevec - \linmap (\repvec_{\samplevec}) } = \max_{z \in S_0 (\linmap, 1)} \inprod{\samplevec - \linmap (\repvec_{\samplevec}) }{z} = 0
\]
One the other hand,  if \( \dualnorm{\samplevec - \linmap (\repvec_{\samplevec}) } = 0 \), Lemma \ref{lemma:separation-dummy}(i) implies that \( \regulizer = 0 \). Moreover, since \( S_0 (\linmap, 1) \) is an absorbing set to \( \image(\linmap) \), we conclude from the definition \eqref{eq:dual-norm-definition} of the dual function that \( \inprod{\samplevec - \linmap (\repvec_{\samplevec}) }{z} = 0 \) for all \( z \in \image (\linmap) \). Furthermore, since \( \linmap (\repvec_{\samplevec}) \in \image (\linmap) \) it implies from the orthogonality principle that \( \pi_{\linmap} (\samplevec) = \linmap (\repvec_{\samplevec}) \). Consequently,
\[
\min_{z \in \image (\linmap)} \inpnorm{\samplevec - z} \ = \ \inpnorm{\samplevec - \pi_{\linmap} (\samplevec)} = \ \inpnorm{\samplevec - \linmap (\repvec_{\samplevec}) } \ = \ \error .
\]
In other words, we have \( \nbhood{\samplevec}{\error} \cap \image(\linmap) = \emptyset \). The proof is now complete.
\end{proof}

\begin{lemma}
\label{lemma:optimality-condition-of-hvar-and-separator}
Let the linear map \( \linmap : \R{\dsize} \longrightarrow \hilbert \) and non-negative real numbers \( \error, \regulizer \) be given, and \( \samplevec \in \hilbert \setminus \closednbhood{0}{\error} \) be any \( \Depsdelta \)-feasible vector such that \( \separatorset \neq \emptyset \). Then for every \( \separator\opt \in \separatorset \) and \( \hvar\opt \in \frac{1}{(\samplecost)^{1/\horder}} \cdot \codes \), we have
\begin{equation}
\label{eq:optimality-condition-of-hvar-and-separator}
\inprod{\separator\opt}{\linmap (\hvar\opt) } \ = \ \max_{h \in \costatomset} \ \inprod{\separator\opt}{\linmap (\hvar)} \ = \ 1 - \regulizer \inpnorm{\separator\opt} \; .
\end{equation}
\end{lemma}

\begin{proof}
Applying \eqref{eq:maximization-over-atomball-changed-to-hvar} directly to \( \separator\opt \in \separatorset \) gives us
\begin{equation}
\label{eq:optimality-condition-of-hvar-and-separator-1}
\max_{\hvar \in \costatomset} \ \inprod{\separator\opt}{\linmap (\hvar)} \ = \ - \regulizer \inpnorm{\separator\opt} + \dualnorm{\separator\opt} \ = \ 1 - \regulizer \inpnorm{\separator\opt} .
\end{equation}
By denoting \( \mathsf{c}_{\samplevec} = (\samplecost)^{1/\horder} \), we know from \eqref{eq:the-other-inequality } that
\[
\inprod{\separator\opt}{y\opt} = \max_{z \in \atomscaled} \inprod{\separator\opt}{z} = \; \mathsf{c}_{\samplevec} \max_{z \in \atomball} \inprod{\separator\opt}{z} = \; \mathsf{c}_{\samplevec} \dualnorm{\separator\opt} = \mathsf{c}_{\samplevec} .
\]
On substituting for \( y\opt \) by considering \( \repvec_{\samplevec} = \mathsf{c}_{\samplevec} \hvar\opt \) in Remark \ref{remark:alternate-expression-for-intersection-point}, we get
\begin{equation}
\label{eq:dummy-eq-lemma:optimality-condition-of-hvar-and-separator}
\mathsf{c}_{\samplevec} = \inprod{\separator\opt}{y\opt} = \mathsf{c}_{\samplevec} \inprod{\separator\opt}{\linmap (\hvar\opt) } + \frac{\regulizer \mathsf{c}_{\samplevec}}{\error + \regulizer \mathsf{c}_{\samplevec}} \inprod{\separator\opt}{\big( \samplevec - \mathsf{c}_{\samplevec} \linmap (\hvar\opt) \big)} \indicator_{ ]0, +\infty[ } (\regulizer) .
\end{equation}
Whenever \( \regulizer > 0 \) we know from Proposition \ref{proposition:description-of-the-separator-set} that \( \separator\opt \) and \( \big( \samplevec - \mathsf{c}_{\samplevec} \linmap (\hvar\opt) \big) \) are co-linear. Thus, we obtain that:
\[
\inprod{\separator\opt}{ \big( \samplevec - \mathsf{c}_{\samplevec} \linmap (\hvar\opt) \big) } = \inpnorm{\separator\opt} \inpnorm{\samplevec - \mathsf{c}_{\samplevec} \linmap (\hvar\opt) } = (\error + \regulizer \mathsf{c}_{\samplevec}) \inpnorm{\separator\opt} ,
\]
where the last equality follows from Lemma \ref{lemma:uniqueness-of-intersection}. Note that \( \mathsf{c}_{\samplevec} > 0 \) since \( \inpnorm{\samplevec} > \error \). Therefore, cancelling \( \mathsf{c}_{\samplevec} \) throughout in \eqref{eq:dummy-eq-lemma:optimality-condition-of-hvar-and-separator} and simplifying for \( \inprod{\separator\opt}{\linmap (\hvar\opt) } \) yields
\begin{equation}
\label{eq:optimality-condition-of-hvar-and-separator-2}
\inprod{\separator\opt}{\linmap (\hvar\opt) } \; = \; 1 - \Big( \regulizer \inpnorm{\separator\opt} \indicator_{ ]0, +\infty[ } (\regulizer) \Big) \; = \; 1 - \regulizer \inpnorm{\separator\opt} ,
\end{equation}
\eqref{eq:optimality-condition-of-hvar-and-separator} follows at once from \eqref{eq:optimality-condition-of-hvar-and-separator-1} and \eqref{eq:optimality-condition-of-hvar-and-separator-2}.
\end{proof}

\subsection{Convex duality and equivalent reformulations of the linear inverse problem}
\begin{proof}[Proof of Lemma \ref{lemma:coding-problem-guage-function-equivalent}]
If \( \samplevec \) is not \( \Depsdelta \)-feasible, then we know that \( \regulizer = 0 \) and \( \closednbhood{\samplevec}{\error} \cap \image (\linmap) = \emptyset \). Consequently, \( \norm{y}{\linmap} = +\infty \) for all \( y \in \closednbhood{\samplevec}{\error} \). Therefore, the assertion holds since \( \samplecost = +\infty \).

If \( \samplevec \) is \( \Depsdelta \)-feasible, then from Lemma \ref{lemma:uniqueness-of-intersection}, we know that the sets \( \closednbhood{\samplevec}{\error} \) and \( \atomscaled \) intersect at a unique point \( y\opt \). Thus we have
\[
\min_{y \; \in \;  \closednbhood{\samplevec}{\error}} \; \norm{y}{\linmap} \leq \; \norm{y\opt}{\linmap} \leq (\samplecost)^{1/\horder} ,
\]
where the first inequality follows from the fact that \( y\opt \in \closednbhood{\samplevec}{\error} \) and the second one follows from \( y\opt \in (\samplecost)^{1/\horder} \cdot \atomball \) and the definition \eqref{eq:guage-function-definition} of the guage function \( \norm{\cdot}{\linmap} \).

On the one hand, for \( y \in \closednbhood{\samplevec}{\error} \) such that \( \norm{y}{\linmap} = +\infty \), the inequality \( (\samplecost)^{1/\horder} \leq \norm{y}{\linmap} \) holds trivially. On the other hand, for \( y \in \closednbhood{\samplevec}{\error} \) such that \( \norm{y}{\linmap} < +\infty \), we know from the definition \eqref{eq:guage-function-definition} that \( y \in \atomset{\norm{y}{\linmap}} \). Thus, \( \closednbhood{\samplevec}{\error} \cap \atomset{\norm{y}{\linmap}} \neq \emptyset \), and in view of Lemma \ref{lemma:coding-problem-reformulation-1}, we get \( (\samplecost)^{1/\horder} \leq \norm{y}{\linmap} \). Combining the two facts, we conclude 
\[
(\samplecost)^{1/\horder} \; \leq \; \min_{y \; \in \;  \closednbhood{\samplevec}{\error}} \; \norm{y}{\linmap} .
\]
Collecting the two inequalities, \eqref{eq:coding-problem-guage-function-equivalent} follows at once.
\end{proof}

\begin{remark}
The proof of the lemma als implies that \( \norm{y\opt}{\linmap} = (\samplecost)^{1/\horder} \), and therefore, \( y\opt \) is a minimizer in the problem \eqref{eq:coding-problem-guage-function-equivalent}. Furthermore, if \( y' \neq y\opt \) is also a minimizer, then we have \( \norm{y'}{\linmap} = (\samplecost)^{1/\horder} \) and \( y' \in \closednbhood{\samplevec}{\error} \). Then it follows that \( y' \in \atomset{\norm{y'}{\linmap}} = \atomscaled \), and thus \( y' \in \closednbhood{\samplevec}{\error} \cap \atomscaled \). From Lemma \ref{lemma:uniqueness-of-intersection}, we then have \( y' = y\opt \). Which is a contradiction. Thus,
\[
y\opt = \argmin_{y \; \in \; \closednbhood{\samplevec}{\error}} \ \norm{y}{\linmap}.
\]
\end{remark}

\begin{proof}[Proof of Theorem \ref{theorem:coding-problem-equivalent-sup-problem}]
Combining \eqref{eq:coding-problem-guage-function-equivalent} and \eqref{eq:strong-duality-of-guage-function}, we obtain
\begin{equation}
\label{eq:separation-principle-formulation-3}
\begin{aligned}
(\samplecost)^{1/\horder} 
& = \min\limits_{y \in \closednbhood{\samplevec}{\error} } \ \sup_{ \dualnorm{\separator} \leq 1 } \ \inprod{\separator}{y} \\
& \geq \ \sup_{ \dualnorm{\separator} \leq 1 } \ \min\limits_{y \in \closednbhood{\samplevec}{\error} } \ \inprod{\separator}{y} \\
& \geq
\begin{cases}
\begin{aligned}
& \sup_{\separator} && \inprod{\separator}{\samplevec} - \error \inpnorm{\separator}  \\
&\sbjto && \dualnorm{\separator} \leq 1 .
\end{aligned}
\end{cases}
\end{aligned}
\end{equation}
Therefore, \( (\samplecost)^{1/\horder} \) is an upper bound to the optimal value of \eqref{eq:coding-problem-equivalent-sup-problem}. We shall establish the proposition by considering all the possible cases and showing that the upper bound is indeed the supremum.

\textsf{Case 1:  When \( \samplevec \) is not \( \Depsdelta \)-feasible.} We know that this happens only if \( \regulizer = 0 \) and \( \closednbhood{\samplevec}{\error} \cap \image(\linmap) = \emptyset \). Denoting \( \pi_{\linmap} (\samplevec) \) to be the orthogonal projection of \( \samplevec \) onto \( \image(\linmap) \), we have \( \inprod{\samplevec - \pi_{\linmap} (\samplevec)}{z} = 0 \) for every \( z \in \image (\linmap) \). 

Since \( \regulizer = 0 \) we have \( S_0 (\linmap , 1) \subset \image (\linmap) \). Thus, for every \( \alpha \geq 0 \), letting \( \separator'_{\alpha} \define \alpha (\samplevec - \pi_{\linmap} (\samplevec) ) \) we see that \( \inprod{\separator'_{\alpha}}{z} = 0 \) for every \( z \in S_0 (\linmap , 1) \). In other words, we have \( \norm{\separator_{\alpha}}{\linmap}' = 0 \), and therefore, \( \separator'_{\alpha} \) is a feasible point in \eqref{eq:coding-problem-equivalent-sup-problem} for every \( \alpha \geq 0 \). Moreover, since \( \closednbhood{\samplevec}{\error} \cap \image(\linmap) = \emptyset \) we see that \( \inpnorm{\samplevec - \pi_{\linmap} (\samplevec) } \geq \error + \rho \) for some \( \rho > 0 \). Therefore, the objective function of \eqref{eq:coding-problem-equivalent-sup-problem} evaluated at \( \separator_{\alpha} \) satisfies
\[
\begin{aligned}
\inprod{\separator'_{\alpha}}{\samplevec} - \error \inpnorm{\separator'_{\alpha}} 
& = \alpha \Big( \inprod{\samplevec - \pi_{\linmap} (\samplevec)}{\samplevec}  - \error \inpnorm{ \samplevec - \pi_{\linmap} (\samplevec) } \Big) \\
& = \alpha \Big( \inpnorm{\samplevec - \pi_{\linmap} (\samplevec)}^2 + \inprod{\samplevec - \pi_{\linmap} (\samplevec)}{\pi_{\linmap} (\samplevec)}  - \error \inpnorm{ \samplevec - \pi_{\linmap} (\samplevec) } \Big) \\
& = \alpha \inpnorm{ \samplevec - \pi_{\linmap} (\samplevec) } \Big( \inpnorm{\samplevec - \pi_{\linmap} (\samplevec)} - \error \Big) \\
& \geq \alpha (\error + \rho) \rho . 
\end{aligned}
\]
By considering arbitrarily large value of \( \alpha \), we observe that the cost function in \eqref{eq:coding-problem-equivalent-sup-problem} attains arbitrarily large values for \( \separator'_{\alpha} \), i.e., the supremum is \( + \infty \).

\textsf{Case 2: When \( 0 \leq \inpnorm{\samplevec} \leq \error \).} We know that the optimal cost \( \samplecost \) is identically equal to zero, and we shall conclude that so is the value of the supremum in \eqref{eq:coding-problem-equivalent-sup-problem}. Indeed, since \( 0 \in \closednbhood{\samplevec}{\error} \), for every \( \separator \in \hilbert \) we have
\[
\inprod{\separator}{x} - \error \inpnorm{\separator} = \min\limits_{y \in \closednbhood{\samplevec}{\error}} \inprod{\separator}{y} \ \leq \ \inprod{\separator}{0} \ = \ 0 .
\]
Thus, zero is an upper bound for the supremum in \eqref{eq:coding-problem-equivalent-sup-problem}. Moreover, for \( \separator\opt = 0 \), we have \( \dualnorm{\separator\opt} = 0 \) and \( \inprod{\separator\opt}{\samplevec} - \error \inpnorm{\separator\opt} = 0 \). Thus, the value of the supremum is achieved, and \( \separator\opt = 0 \) is an optimal solution.\footnote{It is to be to be noted that whenever \( \inpnorm{\samplevec} = \error \), there could be non-zero optimal solutions, for e.g., \( \separator\opt = \alpha \samplevec \) for every \( \alpha \geq 0 \).}

\textsf{Case 3: When \( \samplevec \) is a \( \Depsdelta \)-feasible, and \( \inpnorm{\samplevec} > \error \) with \( \separatorset \neq \emptyset \).} We know that there exists a \( \separator\opt \in \separatorset \) and the following two conditions hold simultameously:
\begin{align*}
\dualnorm{\separator\opt} \; & = \; 1 , \text{ and} \\
\inprod{\separator\opt}{\samplevec} - \error \inpnorm{\separator\opt} \; & = \; (\samplecost)^{1/\horder} .
\end{align*}
The first equality implies that \( \separator\opt \) is a feasible point to \eqref{eq:coding-problem-equivalent-sup-problem}, and the latter, in conjunction with \eqref{eq:separation-principle-formulation-3} implies that the upper bound of \( (\samplecost)^{1/\horder} \) is achieved at \( \separator\opt \). Thus, \( (\samplecost)^{1/\horder} \) is indeed the optimum value of \eqref{eq:coding-problem-equivalent-sup-problem}, and that every \( \separator\opt \in \separatorset \) is an optimal solution to \eqref{eq:coding-problem-equivalent-sup-problem}.

Conversely, if \( \separator\opt \) is an optimal solution to \eqref{eq:coding-problem-equivalent-sup-problem}, then readily we get \( \inprod{\separator\opt}{\samplevec} - \error \inpnorm{\separator\opt} = (\samplecost)^{1/\horder} \). It suffices to show that \( \dualnorm{\separator\opt} = 1 \). Since \( (\samplecost)^{1/\horder} > 0 \), we have \( \inprod{\separator\opt}{\samplevec} - \error \inpnorm{\separator\opt} > 0 \). Therefore, from the assertions (i) and (ii) of Lemma \ref{lemma:separation-dummy}, we conclude that \( \dualnorm{\separator\opt} > 0 \). Moreover, if \( \dualnorm{\separator\opt} < 1 \), then \( \separator' \define \frac{1}{\dualnorm{\separator\opt}} \separator\opt \) is also a feasible point to \eqref{eq:coding-problem-equivalent-sup-problem}. However, the cost function evaluated at \( \separator' \) satisfies
\[
\inprod{\separator'}{\samplevec} - \error \inpnorm{\separator'} = \frac{1}{\dualnorm{\separator\opt}} \big( \inprod{\separator\opt}{\samplevec} - \error \inpnorm{\separator\opt} \big) > (\samplecost)^{1/\horder} ,
\]
which is a contradiction. Therefore, it follows at once that \( \separator\opt \in \separatorset \).

\textsf{Case 4: When \( \samplevec \) is a \( \Depsdelta \)-feasible vector and \( \inpnorm{\samplevec} > \error \) with \( \separatorset = \emptyset \).} We know from Proposition \ref{proposition:description-of-the-separator-set} that this happens only if \( \regulizer = 0 \) and \( \nbhood{\samplevec}{\error} \cap \image (\linmap) = \emptyset \). Since \( \samplevec \) is \( \Depsdelta \)-feasible, the set \( \closednbhood{\samplevec}{\error} \) intersects \( \image (\linmap) \) only at the point \( \pi_{\linmap} (\samplevec) \) - the orthogonal projection of \( \samplevec \) onto \( \image (\linmap) \). Since no point other than \( \pi_{\linmap} (\samplevec) \) in \( \closednbhood{\samplevec}{\error} \) intersects with \( \image(\linmap) \), the LIP \eqref{eq:coding-problem} reduces to the following:
\[
\begin{cases}
\begin{aligned}
& \minimize_{\repvec \; \in \; \R{\dsize}} && \cost (\repvec) \\
& \sbjto && \linmap (\repvec) = \pi_{\linmap} (\samplevec) ,
\end{aligned}
\end{cases}
\]
which simply is another LIP with parameters \( \pi_{\linmap} (\samplevec) \), \( \linmap \) and \( \error = \regulizer = 0 \). Since, \( \pi_{\linmap} (\samplevec) \in \image (\linmap) \), \( \pi_{\linmap} (\samplevec) \) is \( (\linmap, 0, 0) \)-feasible. Therefore, \( \encodedcost{\linmap}{\samplevec}{\error}{0} = \encodedcost{\linmap}{\pi_{\linmap} (\samplevec)}{0}{0} \) and \( \encodermap{\linmap}{\samplevec}{\error}{0} = \encodermap{\linmap}{\pi_{\linmap} (\samplevec)}{0}{0} \). In addition, from the Proposition \ref{proposition:description-of-the-separator-set} it follows that the set \( \Lambda_0 (\linmap, \pi_{\linmap} (\samplevec), 0) \) is non-empty, and there exists \( \separator' \in \image(\linmap) \) such that the following two conditions hold simultaneously.
\[
   \inprod{\separator'}{\pi_{\linmap} (\samplevec)} = \big( \encodedcost{\linmap}{\pi_{\linmap} (\samplevec)}{0}{0} \big)^{1/\horder} =  \big( \encodedcost{\linmap}{\samplevec}{\error}{0} \big)^{1/\horder} \; \text{and} \; \dualnorm{\separator'} = 1
\]
Using the above facts, we shall first establish that the value \( (\samplecost)^{1/\horder} \) is not just an upper bound but is indeed the supremum in \eqref{eq:coding-problem-equivalent-sup-problem}.

For every \( \alpha \geq 0 \) let \( \separator (\alpha) \define \separator' + \alpha (\samplevec - \pi_{\linmap} (\samplevec)) \). Since the linear functional \( \inprod{\samplevec - \pi_{\linmap} (\samplevec)}{\cdot} \) vanishes on \( \image (\linmap) \), for every \( z \in \image (\linmap) \) we have
\[
\begin{aligned}
& \inprod{\separator (\alpha)}{z} \; = \; \inprod{\separator' }{z} + \alpha \inprod{\samplevec - \pi_{\linmap} (\samplevec)}{z} \; = \; \inprod{\separator' }{z} , \ \text{and therefore,} \\
& \dualnorm{\separator (\alpha)} = \max_{z \in S_0 (\linmap , 1)} \inprod{\separator (\alpha)}{z} = \max_{z \in S_0 (\linmap , 1)} \inprod{\separator'}{z} = \dualnorm{\separator'} = 1 .
\end{aligned}
\]
Thus, \( \separator (\alpha) \) is a feasible point to \eqref{eq:coding-problem-equivalent-sup-problem}, and the cost function evaluated at \( \separator (\alpha) \) satisfies:
\[
\begin{aligned}
\inprod{\separator(\alpha)}{\samplevec} - \error \inpnorm{\separator(\alpha)} & = \inprod{\separator(\alpha)}{\pi_{\linmap} (\samplevec)} + \inprod{\separator(\alpha)}{\samplevec - \pi_{\linmap} (\samplevec)} - \error \inpnorm{\separator(\alpha)} \\
& = \inprod{\separator'}{\pi_{\linmap} (\samplevec)} + \alpha \inpnorm{\samplevec - \pi_{\linmap} (\samplevec)}^2 - \error \sqrt{\inpnorm{\separator'}^2 + \alpha^2 \inpnorm{\samplevec - \pi_{\linmap} (\samplevec)}^2} \\
& = \big( \encodedcost{\linmap}{\samplevec}{\error}{0} \big)^{1/\horder} + \error \Big( \alpha \error - \sqrt{\inpnorm{\separator'}^2 + \alpha^2 \error^2} \Big).
\end{aligned}
\]
Since \( \separator(\alpha) \) is feasible in \eqref{eq:coding-problem-equivalent-sup-problem} for every \( \alpha \geq 0 \), the supremum in \eqref{eq:coding-problem-equivalent-sup-problem} is sandwitched between \( \sup\limits_{\alpha \geq 0} \ \inprod{\separator(\alpha)}{\samplevec} - \error \inpnorm{\separator(\alpha)} \) and the optimal cost \( \big( \encodedcost{\linmap}{\samplevec}{\error}{0} \big)^{1/\horder} \). However, we see that:
\[
\sup\limits_{\alpha \geq 0} \ \inprod{\separator(\alpha)}{\samplevec} - \error \inpnorm{\separator(\alpha)} 
\ \geq \ \lim_{\alpha \to +\infty} \inprod{\separator(\alpha)}{\samplevec} - \error \inpnorm{\separator(\alpha)} \ = \ \big( \encodedcost{\linmap}{\samplevec}{\error}{0} \big)^{1/\horder} .\footnote{For any \( b > 0 \), we see that 
\[
\lim_{\alpha \to +\infty} \left( \alpha - \sqrt{b + \alpha^2} \right) = \lim_{\alpha \to +\infty} \frac{ \left( \alpha - \sqrt{b + \alpha^2} \right)\left( \alpha + \sqrt{b + \alpha^2} \right) }{\left( \alpha + \sqrt{b + \alpha^2} \right)} = \lim_{\alpha \to +\infty} \frac{- b}{\alpha + \sqrt{\alpha^2 + b}} = 0 .
\]}
\]
This implies that the supremum in \eqref{eq:coding-problem-equivalent-sup-problem} is indeed equal to \( ( \encodedcost{\linmap}{\samplevec}{\error}{0} )^{1/\horder} \). 

Now that we know the value of the supremum, it suffices to establish that \eqref{eq:coding-problem-equivalent-sup-problem} does not admit an optimal solution in this case. If there were any \( \separator' \) that is an optimal solution to \eqref{eq:coding-problem-equivalent-sup-problem}, then from the arguments provided in the proof of necessary implication for case 3, it follows that \( \separator' \in \Lambda_0 (\linmap, \samplevec, \error) \). This contradicts the premise \( \Lambda_0 (\linmap, \samplevec, \error) = \emptyset \). Therefore, \eqref{eq:coding-problem-equivalent-sup-problem} admits no solution whenever \( \Lambda_0 (\linmap, \samplevec, \error) = \emptyset \).
\end{proof}

\begin{lemma}
\label{lemma:sup-problem-1}
Let the linear map \( \linmap : \R{\dsize} \longrightarrow \hilbert \), non-negative real numbers \( \error , \regulizer \) and \( \samplevec \in \hilbert \setminus \closednbhood{0}{\error} \) be given. For every \( \hvar \in \costatomset \), consider the optimization problem
\begin{equation}
\label{eq:sup-problem-1}
\begin{cases}
\begin{aligned}
& \sup_{\separator} && \inprod{\separator}{\samplevec} - \error \inpnorm{\separator} \\
& \sbjto &&
\begin{cases}
\inprod{\separator}{\samplevec} - \error \inpnorm{\separator} \; > \; 0 , \\
\inprod{\separator}{\linmap (\hvar)} + \regulizer \inpnorm{\separator} \ \leq \ 1 .
\end{cases}
\end{aligned}
\end{cases}
\end{equation}
\begin{enumerate}[leftmargin = * , label = \rm{(\roman*)}]
 \item The optimal value of \eqref{eq:sup-problem-1} is equal to
\begin{equation}
\label{eq:eta-h-definition}
\dualvar_{\hvar} \define \min \big\{ \theta \geq 0 : \closednbhood{\samplevec}{\error} \cap \closednbhood{\linmap (\theta \hvar)}{\theta \regulizer} \neq \emptyset \big\} .
\end{equation}

\item \( \dualvar_{\hvar} \geq (\samplecost)^{1/\horder} \) and equality holds if and only if \( h \in \frac{1}{(\samplecost)^{1/\horder}} \codes \).

\item \( \dualvar_{\hvar} = +\infty \) if and only if there exists a \( \separator' \in \hilbert \) that simultaneously satisfies the conditions
\begin{itemize}
    \item \( \inprod{\separator'}{\samplevec} - \error \inpnorm{\separator'} > 0 \)
    \item \( \inprod{\separator'}{\linmap (\hvar)} + \regulizer \inpnorm{\separator'} \leq 0 \).
\end{itemize}
\end{enumerate}
\end{lemma}
\begin{proof}
Let the map \( L (\dualvar , \hvar) : [0 , +\infty[ \times \costatomset \longrightarrow [0 , +\infty[ \) be defined by
\begin{equation}
\label{eq:L-map_definition}
L (\dualvar , \hvar) \define
\begin{cases}
\begin{aligned}
& \sup_{\separator} && \Big( \inprod{\separator}{\samplevec} - \error \inpnorm{\separator} \Big) - \dualvar \Big( \inprod{\separator}{\linmap (\hvar)} + \regulizer \inpnorm{\separator} \Big) \\
& \sbjto && \inprod{\separator}{\samplevec} - \error \inpnorm{\separator} > 0 .
\end{aligned}
\end{cases}
\end{equation}
For every \( \dualvar \geq 0 \), let us define the set \( S'(\dualvar) \define \bigcup\limits_{\theta \in [0 , \dualvar]} \closednbhood{\linmap (\theta \hvar)}{\regulizer \theta} \). Clearly \( S'(\dualvar) \) is a convex-compact subset of \( \hilbert \) and monotonic, i.e., \( S'(\dualvar) \subset S'(\dualvar') \) for every \( \dualvar \leq \dualvar' \).

For every \( \hvar \in \costatomset \) and \( \theta \geq 0 \), we observe that \( \closednbhood{\linmap (\theta \hvar)}{(\theta \regulizer)} = \theta \cdot \closednbhood{\linmap (\hvar)}{\regulizer} \). Since the sets \( \closednbhood{\samplevec}{\error} \) and \( \closednbhood{\linmap (\hvar)}{\regulizer} \) are compact, the minimization over \( \theta \geq 0 \) in \eqref{eq:eta-h-definition} is achieved. Therefore, we have \( \closednbhood{\samplevec}{\error} \cap \closednbhood{\linmap ( \dualvar_{\hvar} \hvar) }{( \dualvar_{\hvar} \regulizer)} \neq \emptyset \). On the one hand, for \( 0 \leq \dualvar < \dualvar_{\hvar} \leq + \infty \), we know that the convex sets \( \closednbhood{\samplevec}{\error} \) and \( S'(\dualvar) \) do not intersect. Therefore, there exists a non-zero \( \separator' \in \hilbert \) such that the linear functional \( \inprod{\separator'}{\cdot} \) separates them. In other words, we have
\[
\min_{y \in \closednbhood{\samplevec}{\error}} \; \inprod{\separator'}{y} \ > \max_{z \in S'(\dualvar)} \; \inprod{\separator'}{z} .
\]
Observing the following equalities
\[
\begin{aligned}
\min_{y \in \closednbhood{\samplevec}{\error}} \; \inprod{\separator'}{y} & = \inprod{\separator'}{\samplevec} - \error \inpnorm{\separator'} ,  \quad \text{and} \\
\max_{z \in S'(\dualvar)} \; \inprod{\separator'}{z} & = \max \Big\{ 0, \max_{z \in \closednbhood{\linmap (\dualvar \hvar)}{\regulizer \dualvar}} \; \inprod{\separator'}{z} \Big\} \\
& = \max \Big\{ 0 , \; \dualvar \Big( \inprod{\separator'}{\linmap (\hvar)} + \regulizer \inpnorm{\separator'} \Big) \Big\} ,
\end{aligned}
\]
we get
\begin{equation}
\begin{aligned}
\inprod{\separator'}{\samplevec} - \error \inpnorm{\separator'} \ > \ \max \Big\{ 0 , \dualvar \Big( \inprod{\separator'}{\linmap (\hvar)} + \regulizer \inpnorm{\separator'} \Big) \Big\} .
\end{aligned}
\end{equation}
It follows at once that for every \( \alpha \geq 0 \), \( \separator'_{\alpha} \define \alpha \separator' \) is a feasible point in \eqref{eq:L-map_definition}, and thus, we have 
\[
\begin{aligned}
L(\dualvar , \hvar) \ & \geq \ \sup_{\alpha \geq 0} \quad \Big( \inprod{\separator'_{\alpha}}{\samplevec} - \error \inpnorm{\separator'_{\alpha}} \Big) - \dualvar \Big( \inprod{\separator'_{\alpha}}{\linmap (\hvar)} + \regulizer \inpnorm{\separator'_{\alpha}} \Big) \\
& = \ \Big( \inprod{\separator'}{\samplevec} - \error \inpnorm{\separator'} \Big) - \dualvar \Big( \inprod{\separator'}{\linmap (\hvar)} + \regulizer \inpnorm{\separator'} \Big) \ \Big( \sup_{\alpha \geq 0} \; \alpha \Big) \\
& = \ +\infty.
\end{aligned}
\]

On the other hand, for \( \dualvar_{\hvar} \leq \dualvar < +\infty \), we know that \( \closednbhood{\samplevec}{\error} \cap S'(\dualvar) \neq \emptyset \). Due to convexity, we know that for every \( \separator \in \hilbert \), we have
\[
\inprod{\separator}{\samplevec} - \error \inpnorm{\separator} = \min_{y \in \closednbhood{\samplevec}{\error}} \; \inprod{\separator'}{y} \ \leq \max_{z \in S'(\dualvar)} \; \inprod{\separator'}{z} = \max \Big\{ 0 , \dualvar \Big( \inprod{\separator}{\linmap (\hvar)} + \regulizer \inpnorm{\separator} \Big) \Big\} .
\]
Therefore, for every \( \separator \) such that \( \inprod{\separator}{\samplevec} - \error \inpnorm{\separator} > 0 \), we obtain that \( \inprod{\separator}{\linmap (\hvar)} + \regulizer \inpnorm{\separator} > 0 \) and
\[
\inprod{\separator}{\samplevec} - \error \inpnorm{\separator} \; \leq \;  \dualvar \Big( \inprod{\separator}{\linmap (\hvar)} + \regulizer \inpnorm{\separator} \Big) .
\]
By taking the supremum over all \( \separator \), we obtain \( L(\dualvar , \hvar) \leq 0 \). However, by picking any \( \separator \) such that \( \inprod{\separator}{\samplevec} - \error \inpnorm{\separator} > 0 \), and defining \( \separator_{\alpha} \define \alpha \separator \) for every \( \alpha > 0 \), we immediately see that \( \inprod{\separator_{\alpha}}{\samplevec} - \error \inpnorm{\separator_{\alpha}} > 0 \) and
\[
0 = \lim_{\alpha \to 0} \; \Big( \inprod{\separator_{\alpha}}{\samplevec} - \error \inpnorm{\separator_{\alpha}} \Big) - \dualvar \Big( \inprod{\separator_{\alpha}}{\linmap (\hvar)} + \regulizer \inpnorm{\separator_{\alpha}} \Big) .
\]
Therefore, \( L (\dualvar , \hvar) = 0 \). Summarizing, we have:\footnote{It is to be noted that if \( \dualvar_{\hvar} = +\infty \), then \( L(\dualvar , \hvar) = +\infty \) for every \( \dualvar \in [0 , +\infty[ \).}
\[
L (\dualvar , \hvar) =
\begin{cases}
\begin{aligned}
 +\infty & \quad \text{if } \; 0 \leq \dualvar < \dualvar_{\hvar} \\
 0 & \quad \text{if } \; \dualvar_{\hvar} \leq \dualvar < \infty .
\end{aligned}
\end{cases}
\]

Let us consider the Lagrange dual of \eqref{eq:sup-problem-1}, which is written in the following inf-sup formulation.
\begin{equation}
\label{eq:sup-problem-primal-dual}
\begin{cases}
\begin{aligned}
& \inf\limits_{\dualvar \; \geq \; 0} \; \sup_{\separator }  && \inprod{\separator}{\samplevec} - \error \inpnorm{\separator} - \dualvar \Big( \regulizer \inpnorm{\separator} + \inprod{\separator}{\linmap (\hvar)} \; - \; 1 \Big) \\
& \sbjto &&  \inprod{\separator}{\samplevec} - \error \inpnorm{\separator} >  0 .
\end{aligned}
\end{cases}
\end{equation}
Solving for the supremum over \( \separator \), the inf-sup problem \eqref{eq:sup-problem-primal-dual} reduces to \( \inf\limits_{\dualvar \geq 0} \; \dualvar + L (\dualvar , \hvar) \). It is immediate that the optimal value of the inf-sup problem \eqref{eq:sup-problem-primal-dual} is equal to \( \dualvar_{\hvar} \). 

We observe that the optimization problem \eqref{eq:sup-problem-1}, is a convex program. Moreover, since \( \inpnorm{\samplevec} > \error \), we see that \( \separator' \define \alpha \samplevec \) is a strictly feasible point in \eqref{eq:sup-problem-1} for every \( 0 < \alpha < \frac{1}{\inprod{\samplevec}{\linmap (\hvar)} + \regulizer \inpnorm{\samplevec}} \). Therefore, strong duality holds for the convex problem \eqref{eq:sup-problem-1}, and the optimal value of \eqref{eq:sup-problem-1} is indeed equal to \( \dualvar_{\hvar} \). This establishes the assertion (i) of the lemma.

Since \( \closednbhood{\linmap (\hvar)}{\regulizer} \subset \atomball \) we see that 
\[
\closednbhood{\linmap (\dualvar_{\hvar} \hvar)}{\regulizer \dualvar_{\hvar}} \; = \; \dualvar_{\hvar} \cdot  \closednbhood{\linmap (\hvar)}{\regulizer} \; \subset \; \dualvar_{\hvar} \cdot \atomball \; \subset \; \atomset{\dualvar_{\hvar}} .
\]
Combining this with the fact that \( \closednbhood{\samplevec}{\error} \cap \closednbhood{\linmap ( \dualvar_{\hvar} \hvar) }{( \dualvar_{\hvar} \regulizer)} \neq \emptyset \), we immediately infer that \( \closednbhood{\samplevec}{\error} \cap \atomset{\dualvar_{\hvar}} \neq \emptyset \). In view of Lemma \ref{lemma:coding-problem-reformulation-1}, we have \( \dualvar_{\hvar} \geq (\samplecost)^{1 / \horder} \). It is a straight forward exercise to verify that \( \dualvar_{\hvar'} = (\samplecost)^{1 / \horder} \) for some \( \hvar' \in \costatomset \), if and only if \( (\samplecost)^{1/\horder} \hvar' \in \codes \). This establishes assertion (ii) of the lemma.

If there exists a \( \separator' \in \hilbert \) such that the conditions \( \inprod{\separator'}{\samplevec} - \error \inpnorm{\separator'} > 0 \) and \( \inprod{\separator'}{\linmap (\hvar)} + \regulizer 
\inpnorm{\separator'} \leq 0 \) hold simultaneously. Then for every \( \alpha > 0 \), \( \separator_{\alpha} \define \alpha \separator' \) is a feasible point in \eqref{eq:sup-problem-1}. Therefore, we have
\[
\dualvar_{\hvar} \ \geq \ \sup_{\alpha > 0} \; \inprod{\separator_{\alpha}}{\samplevec} - \error \inpnorm{\separator_{\alpha}} \ = \ +\infty .
\]
Conversely let \( \dualvar_{\hvar} = +\infty \), then we know that the compact-convex set \( \closednbhood{\samplevec}{\error} \) does not intersect with the closed convex-cone \( S' \define \bigcup\limits_{\theta \in [0 , +\infty[} \closednbhood{\linmap (\theta \hvar)}{\regulizer \theta} \). Since one of the sets involved is compact, there exists a \( \separator' \in \hilbert \) such that the liear functional \( \inprod{\separator'}{\cdot} \) separates these sets \emph{strictly}. Thus, we have 
\[
\max_{z \in S'} \ \inprod{\separator'}{z} \ < \ \min_{y \in \closednbhood{\samplevec}{\error}} \inprod{\separator'}{y} \ = \ \inprod{\separator'}{\samplevec} - \error \inpnorm{\separator'} .
\]
We note that the quantity \( \inprod{\separator'}{\samplevec} - \error \inpnorm{\separator'} \) is a minimum of a linear functional over a compact set, and thus finite. On the contrary, \( \max\limits_{z \in S'} \ \inprod{\separator'}{z} \) is a maximum of the linear functional \( \separator' \) over the cone \( S' \). Therefore, it can be either \( 0 \) or \( +\infty \). However, since \( \inprod{\separator'}{\samplevec} - \error \inpnorm{\separator'} \) is an upper bound to this maximum, we have \( 0 \; = \;\max\limits_{z \in S'} \ \inprod{\separator'}{z} \). Therefore, we get \( \inprod{\separator'}{\samplevec} - \error \inpnorm{\separator'} > 0 \), and since \( \closednbhood{\linmap (\hvar)}{\regulizer} \subset S' \) we also have \( \inprod{\separator'}{\linmap (\hvar)} + \regulizer \inpnorm{\separator'} \leq 0 \). This completes the proof.
\end{proof}

\begin{lemma}
\label{lemma:sup-problem-2}
Let the linear map \( \linmap \), real numbers \( \error , \regulizer \geq 0 \), \( q \in ]0, 1[ \), \( r > 0 \) and \( \samplevec \in \hilbert \setminus \closednbhood{\samplevec}{\error} \) be given. For every \( \hvar \in \costatomset \), let us consider the following optimization problem:
\begin{equation}
\label{eq:sup-problem-2}
\begin{cases}
\begin{aligned}
& \sup_{\separator } && r \Big( \inprod{\separator}{\samplevec} -  \error \inpnorm{\separator} \Big)^q - \Big( \regulizer \inpnorm{\separator} + \inprod{\separator}{\linmap (\hvar)} \Big) \\
& \sbjto &&  \inprod{\separator}{\samplevec} - \error \inpnorm{\separator} > 0 .
\end{aligned}
\end{cases}
\end{equation}
The optimal value of \eqref{eq:sup-problem-2} is \( s(r, q) \; \dualvar_{\hvar}^\frac{q}{1 - q} \), where \( \dualvar_{\hvar} \) is as defined in \eqref{eq:eta-h-definition} and \( s(r, q) \) is some constant .\footnote{ \( s(r, q) \define \Big( (1 - q) (q^q r )^{\frac{1}{1 - q}} \Big) \; \) }
\end{lemma}
\begin{proof}
We begin by considering the case when \( \dualvar_{\hvar} < +\infty \). From the assertion (iii) of Lemma \ref{lemma:sup-problem-1}, it follows that \( \inprod{\separator}{\linmap (\hvar)} + \regulizer \inpnorm{\separator} > 0 \) for every \( \separator \in \hilbert \) satisfying \( \inprod{\separator}{\samplevec} - \error \inpnorm{\separator} > 0 \). Then, the optimization problem can be equivalently written as 
\[
\begin{cases}
\begin{aligned}
& \sup_{\separator , \; \alpha > 0 } && r \Big( \inprod{\separator}{\samplevec} -  \error \inpnorm{\separator} \Big)^q - \alpha \\
& \sbjto && 
\begin{cases}
\inprod{\separator}{\linmap (\hvar)} + \regulizer \inpnorm{\separator} = \alpha \\
\inprod{\separator}{\samplevec} - \error \inpnorm{\separator} > 0 .
\end{cases}
\end{aligned}
\end{cases}
\]
Redefining new variables \( \separator' \define \frac{1}{\alpha} \separator \), the above optimization problem is written as
\[
\begin{cases}
\begin{aligned}
& \sup_{\separator' , \; \alpha > 0 } &&  \alpha^q r \Big( \inprod{\separator'}{\samplevec} -  \error \inpnorm{\separator'} \Big)^q - \alpha \\
& \sbjto && 
\begin{cases}
\inprod{\separator'}{\linmap (\hvar)} + \regulizer \inpnorm{\separator'} = 1 \\
\inprod{\separator'}{\samplevec} - \error \inpnorm{\separator'} > 0 .
\end{cases}
\end{aligned}
\end{cases}
\]
By keeping a feasible \( \separator' \) fixed, one can explicitly optimize over \( \alpha > 0 \). In fact, for any \( r' > 0 \) we know that
\[
\sup\limits_{\alpha > 0} \ \big(r' \alpha^q - \alpha \big) \ = \ (r')^{ \frac{1}{1 - q} } q^{ \frac{q}{1 - q} } (1 - q) .
\]
Substituting \( r' = r \big( \inprod{\separator'}{\samplevec} -  \error \inpnorm{\separator'} \big)^q \), we see that \eqref{eq:sup-problem-2} simplifies to
\[
\begin{cases}
\begin{aligned}
& \sup_{\separator' } &&  s(r, q) \Big( \inprod{\separator'}{\samplevec} -  \error \inpnorm{\separator'} \Big)^{ \frac{q}{1 - q} } \\
& \sbjto && 
\begin{cases}
\inprod{\separator'}{\linmap (\hvar)} + \regulizer \inpnorm{\separator'} = 1 \\
\inprod{\separator'}{\samplevec} - \error \inpnorm{\separator'} > 0 .
\end{cases}
\end{aligned}
\end{cases}
\]
Since, \( \dualvar_{\hvar} < +\infty \), we know that \( \inprod{\separator'}{\linmap (\hvar)} + \regulizer \inpnorm{\separator'} > 0 \) for every \( \separator' \) satisfying \( \inprod{\separator'}{\samplevec} - \error \inpnorm{\separator'} > 0 \). Moreover, if \( \inprod{\separator'}{\linmap (\hvar)} + \regulizer \inpnorm{\separator'} < 1 \) also holds for \( \separator' \), we see that its scaled version \( \separator'' \define \frac{1}{\inprod{\separator'}{\linmap (\hvar)} + \regulizer \inpnorm{\separator'}} \separator' \), satisfies
\[
\begin{aligned}
& \inprod{\separator''}{\samplevec} - \error \inpnorm{\separator''} \; > \; \inprod{\separator'}{\samplevec} - \error \inpnorm{\separator'} \quad \text{and} \\
& \inprod{\separator''}{\linmap (\hvar)} + \regulizer \inpnorm{\separator''} = 1 .
\end{aligned}
\]
Therefore, the equality constraint \( \inprod{\separator'}{\linmap (\hvar)} + \regulizer \inpnorm{\separator'} = 1 \) can be relaxed to an inequality without changing the value of the supremum. Thus, we obtain the following problem equivalent to \eqref{eq:sup-problem-2}.
\[
\begin{cases}
\begin{aligned}
& \sup_{\separator' } &&  s(r, q) \Big( \inprod{\separator'}{\samplevec} -  \error \inpnorm{\separator'} \Big)^{ \frac{q}{1 - q} } \\
& \sbjto && 
\begin{cases}
\inprod{\separator'}{\linmap (\hvar)} + \regulizer \inpnorm{\separator'} \leq 1 \\
\inprod{\separator'}{\samplevec} - \error \inpnorm{\separator'} > 0 .
\end{cases}
\end{aligned}
\end{cases}
\]
Finally, we observe that \( [0 , \infty[ \ni (\cdot) \longmapsto (\cdot)^{\frac{q}{1 - q}} \in [0 , +\infty[ \) is an increasing function for every \( q \in ] 0 , 1 [ \). Then it follows at once from Lemma \ref{lemma:sup-problem-1} that the optimal value of \eqref{eq:sup-problem-2} is equal to \( s(r,q) (\dualvar_{\hvar})^{\frac{q}{1 - q}} \).
\end{proof}

\begin{proof}[Proof of Theorem \ref{theorem:coding-problem-primal-dual}]
Solving for the supremum over \( \separator \) for every \( \hvar \in \costatomset \) in the min-sup problem \eqref{eq:coding-problem-primal-dual}, we deduce from Lemma \ref{lemma:sup-problem-2} that \eqref{eq:coding-problem-primal-dual} reduces to 
\[
\min_{\hvar \; \in \; \costatomset} \; s(r, q) \; \dualvar_{\hvar}^{ \frac{q}{1 - q} } .
\]
Since \( ]0 , +\infty[ \ni \dualvar \longmapsto \dualvar^{ \frac{q}{1 - q} } \) is an increasing function for every \( q \in ]0 , 1[ \),  in view of the assertion (ii) of Lemma \ref{lemma:sup-problem-1}, we conclude that the minimization over the variable \( \hvar \) is achieved at \( \hvar\opt \) such that \( \dualvar_{\hvar\opt} = (\samplecost)^{ 1 / \horder } \). Therefore, the optimal value of the min-sup problem \eqref{eq:coding-problem-primal-dual} is equal to \( s(r,q) \; (\samplecost)^{ \frac{q}{\horder (1 - q)} } \) and the set of minimizers is \( \frac{1}{(\samplecost)^{ 1 / \horder }} \codes \).
This establishes the assertions (i) and (ii)-(a) of the theorem.

\noindent \textsf{Necessary condition for \( (\hvar\opt , \separator\opt) \) to be a saddle point solution.} \newline
Suppose that \( (\hvar\opt , \separator\opt ) \in \costatomset \times \hilbert \) is a saddle point solution to the min-sup problem \eqref{eq:coding-problem-primal-dual}. Then necessarily, we have
\begin{equation*}
\hvar\opt \in \;
\argmin\limits_{\hvar \; \in \; \costatomset}
\begin{cases}
\begin{aligned}
& \sup_{\separator } \ \ &&  r \Big( \inprod{\separator}{\samplevec} -  \error \inpnorm{\separator} \Big)^q - \Big( \regulizer \inpnorm{\separator} + \inprod{\separator}{\linmap (\hvar)} \Big) \\
& \sbjto &&  \inprod{\separator}{\samplevec} - \error \inpnorm{\separator} > 0 ,
\end{aligned}
\end{cases}
\end{equation*}
which implies that \( \hvar\opt \in \frac{1}{(\samplecost)^{1/\horder}} \cdot \codes \). Moreover, we also have 
\begin{equation}
\separator\opt \in 
\begin{cases}
\begin{aligned}
& \argmax_{\separator } \ \ && \min\limits_{\hvar \in \costatomset}
\begin{cases}
r \Big( \inprod{\separator}{\samplevec} -  \error \inpnorm{\separator} \Big)^q - \Big( \regulizer \inpnorm{\separator} + \inprod{\separator}{\linmap (\hvar)} \Big)
\end{cases}
\\
& \sbjto &&  \inprod{\separator}{\samplevec} - \error \inpnorm{\separator} > 0 .
\end{aligned}
\end{cases}
\end{equation}
The minimization over \( \hvar \) can be solved explicitly, and simplifying using \eqref{eq:maximization-over-atomball-changed-to-hvar}, we have
\[
\separator\opt \in 
\begin{cases}
\begin{aligned}
& \argmax_{\separator } \ \ && 
r \Big( \inprod{\separator}{\samplevec} -  \error \inpnorm{\separator} \Big)^q - \dualnorm{\separator} \\
& \sbjto &&  \inprod{\separator}{\samplevec} - \error \inpnorm{\separator} \geq 0 .
\end{aligned}
\end{cases}
\]
By defining the new variables \( \alpha \define \dualnorm{\separator} \), and \( \separator' \define \frac{1}{\dualnorm{\separator}} \separator \), and writing the above optimization problem in terms of the variables \( (\separator' , \alpha) \), we obtain
\begin{equation}
\label{eq:optimal-separator-inclusion}
\Big( \frac{\separator\opt}{\dualnorm{\separator\opt}} , \; \dualnorm{\separator\opt} \Big) \; \in 
\begin{cases}
\begin{aligned}
& \argmax_{( \separator' , \; \alpha ) } && r \alpha^q \Big( \inprod{\separator'}{\samplevec} -  \error \inpnorm{\separator'} \Big)^q - \alpha \\
& \sbjto && 
\begin{cases}
\alpha > 0 , \\
\dualnorm{\separator'} = 1 , \\
\inprod{\separator'}{\samplevec} - \error \inpnorm{\separator'} > 0 .
\end{cases}
\end{aligned}
\end{cases}    
\end{equation}
Observe that for every feasible \( \separator' \), the optimization over the variable \( \alpha \) can be solved explicitly. Then from arguments similar to the ones provided in the proof of Lemma \ref{lemma:sup-problem-2}, we conclude that
\[
\begin{aligned}
\frac{1}{\dualnorm{\separator\opt}} \separator\opt & \in
\begin{cases}
\begin{aligned}
& \argmax_{\separator' } &&
\max_{\alpha > 0}
\begin{cases}
r \alpha^q \Big( \inprod{\separator'}{\samplevec} -  \error \inpnorm{\separator'} \Big)^q - \alpha
\end{cases}
\\
& \sbjto && 
\begin{cases}
\dualnorm{\separator'} = 1 \\
\inprod{\separator'}{\samplevec} - \error \inpnorm{\separator'} > 0 ,
\end{cases}
\end{aligned}
\end{cases}
\\
& =
\begin{cases}
\begin{aligned}
& \argmax_{\separator' } && s(r,q) \; \Big( \inprod{\separator'}{\samplevec} -  \error \inpnorm{\separator'} \Big)^{\frac{q}{1 - q}} \\
& \sbjto && 
\begin{cases}
\dualnorm{\separator'} = 1 \\
\inprod{\separator'}{\samplevec} - \error \inpnorm{\separator'} > 0 ,
\end{cases}
\end{aligned}
\end{cases}
\\
& =
\begin{cases}
\begin{aligned}
& \argmax_{\separator' } && \inprod{\separator'}{\samplevec} -  \error \inpnorm{\separator'} \\
& \sbjto && 
\begin{cases}
\dualnorm{\separator'} \leq 1 \\
\inprod{\separator'}{\samplevec} - \error \inpnorm{\separator'} > 0 ,
\end{cases}
\end{aligned}
\end{cases}
\\
& = \separatorset.\footnotemark
\end{aligned}
\footnotetext{Since \( \inpnorm{\samplevec} > \error \), we know that the optimal value achieved in \eqref{eq:coding-problem-equivalent-sup-problem} is positive. Therefore, adding the additional constraint \( \inprod{\separator'}{\samplevec} - \error \inpnorm{\separator'} > 0  \) neither changes the optimizer nor the optimal value. Then the last last equality follows from Theorem \ref{theorem:coding-problem-equivalent-sup-problem}.}
\]
Since every \( \separator \in \separatorset \) satisfies \( \inprod{\separator}{\samplevec} - \error \inpnorm{\separator} = \samplecost^{\frac{1}{\horder}}\), we conclude from \eqref{eq:optimal-separator-inclusion} that the following also holds.
\[
\begin{aligned}
\dualnorm{\separator\opt} & \in \; \argmax_{\alpha \; > \; 0} \Big\{ r \alpha^q \samplecost^{\frac{q}{\horder}} - \alpha
\\
& = ( r q )^{\frac{1}{1 - q}} \big( \samplecost \big)^{\frac{q}{\horder (1 - q)}} .
\end{aligned}
\]
Therefore, \( \separator\opt \in ( r q )^{\frac{1}{1 - q}} \big( \samplecost \big)^{\frac{q}{\horder (1 - q)}} \cdot \separatorset \), and the necessary conditions hold.

\noindent \textsf{Sufficient condition for \( (\hvar\opt , \separator\opt) \) to be a saddle point solution.} \newline
Since \( 0 < ( r q )^{\frac{1}{1 - q}} \big( \samplecost \big)^{\frac{q}{\horder (1 - q)}} \), we conclude from \eqref{eq:optimality-condition-of-hvar-and-separator}, that \( \inprod{\separator\opt}{\linmap (\hvar\opt) } = \max\limits_{\hvar \in \costatomset} \inprod{\separator\opt}{\linmap (\hvar)} \). Then it immediately follows that
\begin{equation}
\label{eq:saddle-point-condition-hvar}
\hvar\opt \in \argmin_{\hvar \in \costatomset} \ r \Big( \inprod{\separator\opt}{\samplevec} -  \error \inpnorm{\separator\opt} \Big)^q - \Big( \regulizer \inpnorm{\separator\opt} + \inprod{\separator\opt}{\linmap (\hvar)} \Big) . 
\end{equation}

From Lemma \ref{lemma:sup-problem-1}(ii), we note that \( \dualvar_{\hvar\opt} = \samplecost^{1/{\horder}} \). Therefore, from Lemma \ref{lemma:sup-problem-2} we have
\[
s(r,q) \samplecost^{ \frac{q}{\horder (1 - q)} } = 
\begin{cases}
\begin{aligned}
& \sup\limits_{\separator} &&  r \Big( \inprod{\separator}{\samplevec} -  \error \inpnorm{\separator} \Big)^q - \Big( \regulizer \inpnorm{\separator} + \inprod{\separator}{\linmap (\hvar\opt}) \Big) \\
& \sbjto &&  \inprod{\separator}{\samplevec} - \error \inpnorm{\separator} > 0 .
\end{aligned}
\end{cases}
\]
Moreover, from \eqref{eq:optimality-condition-of-hvar-and-separator} and the Definition \ref{def:optimal-separator-set} of the set \( \separatorset \), it is a straightforward exercise to verify that
\[
s(r,q) \samplecost^{ \frac{q}{\horder (1 - q)} } = r \Big( \inprod{\separator\opt}{\samplevec} -  \error \inpnorm{\separator\opt} \Big)^q - \Big( \regulizer \inpnorm{\separator\opt} + \inprod{\separator\opt}{\linmap (\hvar)} \Big) .
\]
Since it is obvious from the Definition \ref{def:optimal-separator-set} that \( \inprod{\separator\opt}{\samplevec} - \error \inpnorm{\separator\opt} > 0 \), we get at once that
\begin{equation}
\label{eq:saddle-point-condition-separator}
\separator\opt \in \;
\begin{cases}
\begin{aligned}
& \argmax\limits_{\separator} &&  r \Big( \inprod{\separator}{\samplevec} -  \error \inpnorm{\separator} \Big)^q - \Big( \regulizer \inpnorm{\separator} + \inprod{\separator}{\linmap (\hvar\opt) } \Big) \\
& \sbjto &&  \inprod{\separator}{\samplevec} - \error \inpnorm{\separator} > 0 .
\end{aligned}
\end{cases}
\end{equation}
Collecting \eqref{eq:saddle-point-condition-hvar} and \eqref{eq:saddle-point-condition-separator}, we conclude that \( (\hvar\opt , \separator\opt) \) is indeed a saddle point solution to \eqref{eq:coding-problem-primal-dual}. The proof is now complete.
\end{proof}

\section{Conclusion}
In this article, the convex duality of the linear inverse problems has been studied and equivalent convex-concave min-max formulation has been provided. Such a reformulation provides a simple and effective generic algorithm to solve LIPs, and furthermore, is crucial for solving the error constrained dictionary learning problem.

\vskip 0.2in
\bibliographystyle{alpha}
\bibliography{ref}

\end{document}